\documentclass[nohyperref]{article}

\usepackage{microtype}
\usepackage{graphicx}
\usepackage{subcaption}
\usepackage{wrapfig}
\usepackage{booktabs} 
\usepackage{multirow}
\usepackage{multicol}

\usepackage{hyperref}

\usepackage[accepted]{icml2022}

\usepackage{amsmath}
\usepackage{amssymb}
\usepackage{mathtools}
\usepackage{amsthm}
 
\usepackage[capitalize,noabbrev]{cleveref}

\theoremstyle{plain}
\newtheorem{theorem}{Theorem}[section]
\newtheorem{proposition}[theorem]{Proposition}

\theoremstyle{definition}

\theoremstyle{remark}

\newcommand{\indep}{\perp \!\!\! \perp}
\newcommand{\bigCI}{\mathrel{\text{\scalebox{1.07}{$\perp\mkern-10mu\perp$}}}}

\usepackage[textsize=tiny]{todonotes}
\usepackage{colortbl}

\icmltitlerunning{From Noisy Prediction to True Label: Noisy Prediction Calibration via Generative Model}

\begin{document}

\twocolumn[
\icmltitle{From Noisy Prediction to True Label: \\ Noisy Prediction Calibration via Generative Model}



\icmlsetsymbol{equal}{\textbf{*}}

\begin{icmlauthorlist}
\icmlauthor{HeeSun Bae}{equal,sch}
\icmlauthor{Seungjae Shin}{equal,sch}
\icmlauthor{Byeonghu Na}{sch}
\icmlauthor{JoonHo Jang}{sch}
\icmlauthor{Kyungwoo Song}{sch2}
\icmlauthor{Il-Chul Moon}{sch,comp}
\end{icmlauthorlist}

\icmlaffiliation{comp}{Summary.AI, Daejeon, Republic of Korea}
\icmlaffiliation{sch}{Industrial and Systems Engineering, Korea Advanced Institute of Science and Technology (KAIST), Daejeon, Republic of Korea}
\icmlaffiliation{sch2}{Department of AI, University of Seoul, Seoul, Republic of Korea}

\icmlcorrespondingauthor{Il-Chul Moon}{icmoon@kaist.ac.kr}

\icmlkeywords{Noisy label, Deep generative model, Machine Learning, ICML}

\vskip 0.3in
]



\printAffiliationsAndNotice{\icmlEqualContribution} 

\begin{abstract}
Noisy labels are inevitable yet problematic in machine learning society. It ruins the generalization of a classifier by making the classifier over-fitted to noisy labels. Existing methods on noisy label have focused on modifying the classifier during the training procedure. It has two potential problems. First, these methods are not applicable to a pre-trained classifier without further access to training. Second, it is not easy to train a classifier and regularize all negative effects from noisy labels, simultaneously. We suggest a new branch of method, \textit{Noisy Prediction Calibration} (NPC) in learning with noisy labels. Through the introduction and estimation of a new type of transition matrix via generative model, NPC corrects the noisy prediction from the pre-trained classifier to the true label as a post-processing scheme. We prove that NPC theoretically aligns with the transition matrix based methods. Yet, NPC empirically provides more accurate pathway to estimate true label, even without involvement in classifier learning. Also, NPC is applicable to any classifier trained with noisy label methods, if training instances and its predictions are available. Our method, NPC, boosts the classification performances of all baseline models on both synthetic and real-world datasets. The implemented code is available at https://github.com/BaeHeeSun/NPC.

\end{abstract}

\section{Introduction}
The success of deep neural networks heavily relies on large-scale datasets with annotations \cite{dualT}. Whereas the importance of large-size dataset is unanimous, creating such large-scale dataset is arduous and often prone to human errors in their label annotations \cite{cores}. For instance, the recent utilization of crowd-sourcing \cite{crowd} or search engines \cite{clothing1m} in building the datasets potentially results in the problem of \textit{noisy label} \cite{coteaching, pencil, rel, elr, proselflc}. The model training with noisy label could be detrimental given the fitness of over-parameterized neural network to the training dataset, because such networks are even ready to fit the mislabeled training instances, a.k.a \textit{memorization} \cite{arpit,zhang,rel} problem.

Several studies have suggested resolutions on \textit{memorization} from noisy labels. First, \textit{noise-cleansing} approaches \cite{decoupling,masking,joint,coteaching,coteachingplus,sigua,jocor,lrt,meta,proselflc,fine} focus on segregating the clean data pairs from the corrupted dataset, based on the outputs of noisy classifier, i.e. loss, entropy, and feature alignment. Second, \textit{noise-robust} approaches utilize explicit regularizations \cite{elr,rel} or robust loss functions \cite{gce, sce, apl} to design a classifier robust to the noisy labels. Yet, their modeling assumptions and structures mostly originate from either heuristics or hard assumptions. Another line of researches suggest an explicit formulation on noise patterns, i.e. \textit{transition matrix} $T$ from a true label to a noisy one \cite{forward, dualT, tvr, pdn, csidn}. While methods with $T$ provide rigorous formulation on the label modifications, accurate estimation of $T$ heavily depends on the inference of true label $y$, which is assumed latent from observations \cite{causalNL}.

With the unsolved challenges above, classifiers trained by existing methods are still not robust to label noises, depending on the various noise characteristics \cite{survey_noisy}. It necessitates the modelling of reducing the gap between the prediction of trained classifier and the true latent label. Post-processing can be effective for this objective, in that it is easily applicable to any trained classifier without access to training. Motivated by this spirit, we introduce another branch of solutions on the problem of noisy labels. We propose the algorithm, coined Noisy Prediction Calibration (NPC), which estimates the explicit transition from a noisy prediction to a true latent class via utilizing a deep generative model. NPC operates as a post-processing module to a black-box classifier trained on noisy labels, while previous transition models were applied during the training procedures. Therefore, NPC can expand the scalability of transition models in terms of learning time by far, and NPC inherits wider usability because of its post-processing nature.

For theoretical aspect, we prove that NPC is interchangeable with the transition matrix, which generalizes the modeling framework of NPC. For empirical aspect, NPC significantly boosts the accuracy of classifiers trained with various type of noisy labels without any access to noise ratio. Moreover, NPC captures potentially noisy data instances from the learning with benchmark datasets, i.e. MNIST and Fashion-MNIST, which have been believed to be clean.

\section{Problem Definition}
\label{sec: problem definition}
\subsection{Problem Setup}
Assuming a classification task of $c$ classes, let $\mathcal{X} \subset \mathbb{R}^d$ and $\mathcal{Y} = \left\{1,2,...,c\right\}$ be a $d$-dimensional input space and a label set, respectively. Given the input and the label spaces, the i.i.d. samples from the joint probability distribution $P$ over $\mathcal{X}\times \mathcal{Y}$, $\mathcal{D}=\left\{ (x_i,y_i) \right\}_{i=1}^n$, becomes a classification dataset. Unlike traditional supervised learning, our assumption on the noisy label dictates that only observables are $\tilde{\mathcal{D}}=\left\{ (x_i,\tilde{y_i}) \right\}_{i=1}^n$, which are samples from $\tilde{P}$, potentially different from $P$.

The training objective of a classifier $f$ is to minimize the true risk, $R_{L}(f):= \mathbb{E}_{P}\left[L\left(f\left(x \right),y \right) \right]$, yet the only accessible risk function is the noisy empirical risk of $\tilde{R}_{L}^{emp}(f):=\frac{1}{n}\sum_{i=1}^{n}L\left ( f\left ( x_i \right ),\tilde{y}_i \right )$. Hence, when learning with noisy labels, the objective becomes finding a function that minimizes $R_{L}(f)$ via the learning procedure with $\tilde{R}_{L}^{emp}(f)$.

\subsection{Previous Research on Noisy Label Classification} Neural networks trained with gradient descent can easily fit even random labels \cite{zhang}. To tackle this issue, various works have been introduced for learning with noisy labels. 
One directional approaches include the extraction of reliable clean samples \cite{coteaching, coteachingplus, jocor}, label modification \cite{joint, pencil, lrt, dividemix, meta, proselflc}, introducing noise-robust losses \cite{gce, sce}, and additive regularization \cite{elr,rel}. \footnote{A detailed description of each method is in Appendix \ref{appen: prev noisy}.}
These works usually hinge upon heuristics or assumptions, such as the phenomenon of learning \textit{simple pattern} at the early stage of learning \cite{arpit}. These experimentally justified heuristics require hand-picked hyper-parameters, such as noise ratio, early stopping time, and noisy indication threshold, which are critical for their performance improvement.

Other type of approaches estimate the true class probability by formulating a transition matrix, $T$, as follows:
\begin{equation}
\label{Eq:transition matrix}
T_{kj}(x)=p(\tilde{y}=j|y=k,x) \;\, \text{for} \; \text{all} \; j,k=1,...,c
\end{equation}
Transition matrix, $T$, provides a probabilistic formulation on label transition from true class $y$ to noisy label $\tilde{y}$. Since $p(\tilde{y}|x)=\sum_{k=1}^{c}p(\tilde{y}|y=k,x)p(y=k|x)$, $T$ provides a pathway to the true loss from true label pairs with observable $\tilde{y}$ and $x$ as follows:
\begin{equation}
\label{Eq:revised loss}
E_{\tilde{P}}[L(T(f(x)),\tilde{y})]=R_L(f)
\end{equation}
The benefit of transition matrix methods is the explicit formulation of label distribution modifications. However, estimating $T$ with a latent variable $y$ is not an easy problem.

\subsection{Previous Studies on Transition Matrix Estimation} 
\label{sec:prev trans}
$p(\tilde{y}=j|y=i,x)$ has a large support space if the input space has a large dimension. As an initial solution, \cite{forward,dualT, tvr} suggested the transition matrix at the aggregated level by assuming $p(\tilde{y}=j|y=i,x)=p(\tilde{y}=j|y=i)$ with the instance independence of transition probability. Yet, the estimation gap exists with this assumption because $x$ influences on the transition to a noise label of $\tilde{y}$, i.e. an image of \textit{3} resembling to \textit{5} in MNIST and its potential mislabeling \cite{aaai-idn}. To relax the assumption of instance independence, there are various models on estimating an instance-dependent $T$ \cite{aaai-idn,pdn,csidn, causalNL}, but they still have limitations since they rely on either assumption of part-dependent label noise \cite{pdn} or access to additional information on confidence score \cite{csidn}.

Emphasizing the importance on estimating latent $y$, a recent study \cite{causalNL} proposes an estimation of $T$ by maximizing the likelihood of observable variables, $p(x,\tilde{y})$, through a generative model with the latent $y$. They model a noisy data generative process as Figure \ref{fig:CausalNL}. Their model, CausalNL introduces an auxiliary latent variable of $z$ that is another source of information to generate $x$, other than $y$ so that $z$ and $y$ jointly generate $x$. Eventually, this model learns the generation process of $x$ by Variational Autoencoder (VAE) \cite{vae}, and this serves the maximization of $p(x,\tilde{y})$, which leaves $p(z,x,y,\tilde{y})$ as a side-product that becomes an ingredient to formulate $T$. This generation is inferred via maximizing Evidence Lower Bound (ELBO) of Eq. \ref{eq:ELBOO}.
\begin{equation}
\begin{aligned}
&\text{ELBO}(x,\Tilde{y})= \mathbb{E}_{(z,y)\sim q_{\phi}(z,y|x)} [ \log p_{\theta}(x|y,z) ]\\&+\mathbb{E}_{y\sim q_{\phi}(y|x)}[ \log p_{\theta}(\Tilde{y}|y,x)]-\text{KL}(q_\phi(y|x)||p(y))\\
&-\mathbb{E}_{y\sim q_{\phi}(y|x)}[\text{KL}(q_\phi(z|y,x)||p(z)) ]\\
\end{aligned}
\label{eq:ELBOO}
\end{equation}
This estimation framework may have two shortcomings. First, the generation of $x$, whose data resolution can be high, would not be an appropriate objective for a noisy classification task. This model would infer the generation process with sub-optimal reconstruction, a well-known problem of VAE \cite{vaeblurry}. Therefore, inference on $T$ with VAE would lead to sub-optimal, as well. Second, CausalNL assumes that there is an additional cause of $x$ from $z$, so $z$ and $y$ jointly generate $x$. Given the observed $x$, $z$ and $y$ become dependent by V-structure as in Figure \ref{fig:CausalNL}. For classification, the correlation between $z$ and $y$ needs to be disentangled from $z$ because $y$ needs to model only the pure label information. This disentanglement can be a burden particularly given that the generation of $x$ is not the main goal of classification with noisy label.

Deep generative models for classification with noisy labels have not yet been actively studied. However, the importance of estimating the true class $y$ as latent motivates us to utilize a deep generative model. Having said that, there are differences between CausalNL and our model, NPC, in terms of generative model design and modeling purpose. From the perspective of generative model design, our method differs from CausalNL because we remove the unnecessary generation of $x$ without any auxiliary latent variables. For modelling purposes, CausalNL improves the classifier learning through identifiable estimation of $T$. On the contrary, we introduce and estimate a new type of transition matrix, $H$, which is mainly utilized to post-process the outputs of classifiers trained on noisy datasets. We discuss the advantageous properties of post-processing for noisy classifiers in Section \ref{sec:3.1}. Also, we claim that our introduced transition, $H$, is interchangeable with $T$ in Section \ref{sec: comp with trans}.
\begin{figure}[h]
    \centering
    \includegraphics[width=0.95\columnwidth]{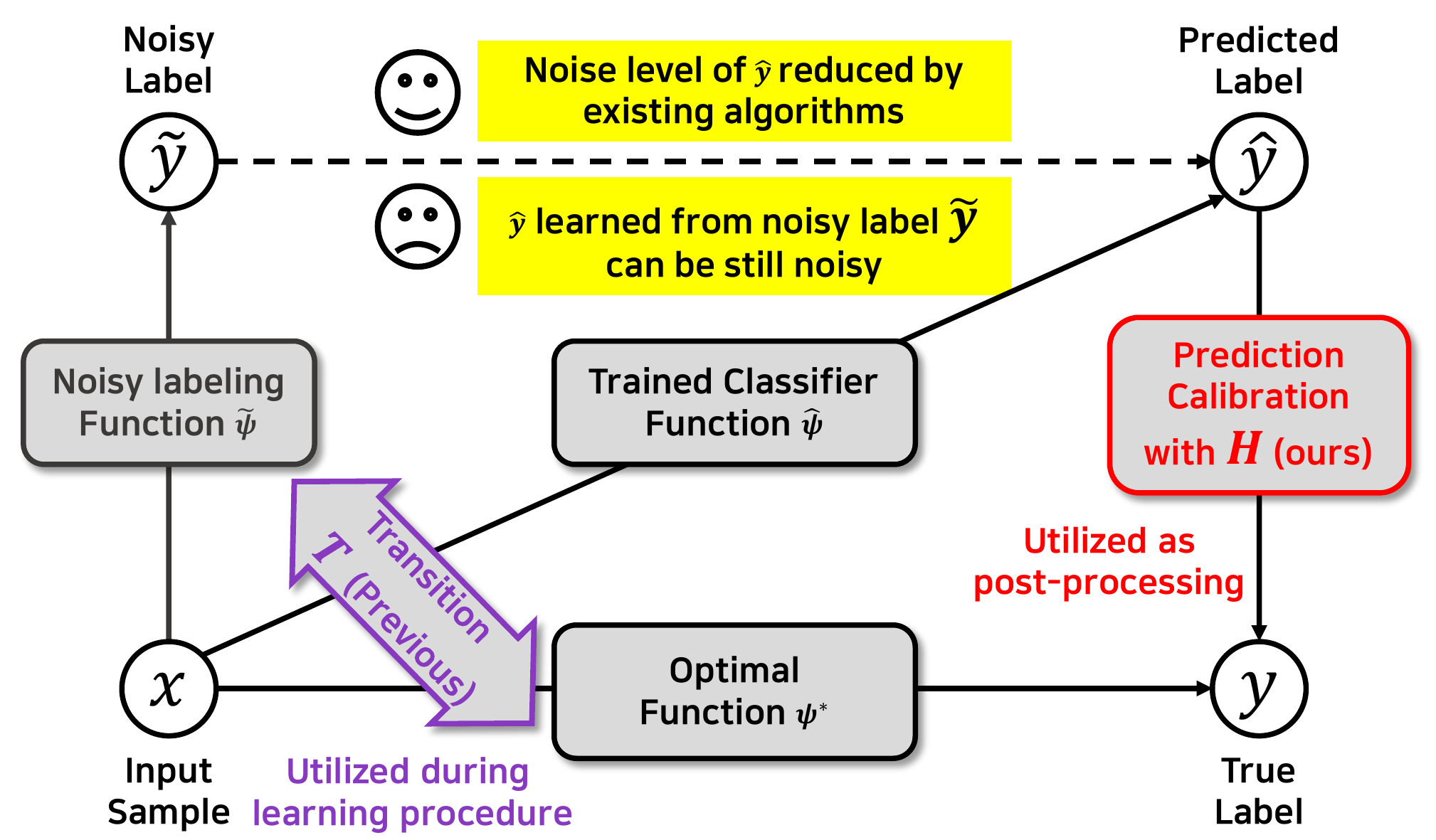}
    \caption{\textbf{Overview.} Existing methods shift the function $\psi$ to estimate the mapping from $x$ to $y$ during the learning procedure. On the contrary, our method, NPC, formulates a prediction calibration procedure which reduces the gap between noisy prediction $\hat{y}$ and true label $y$ under the post-processing framework. NPC successfully boosts the pre-trained classifier performances by leveraging the classifier predictions as noise-reduced inputs.}
    \vskip-0.15in
    \label{fig:method explanation}
\end{figure}
\section{Method}
\label{sec: Method}
This section presents our model, Noisy Prediction Calibration (NPC), by going through its theoretic formulation, probabilistic model structure, and its inference procedure. Additionally, we discuss the relation between NPC and the original transition matrix $T$.

\begin{figure*}
    \centering
    \begin{minipage}{0.162\textwidth}
        \begin{subfigure}[t]{.9\linewidth}
            \includegraphics[width=\textwidth]{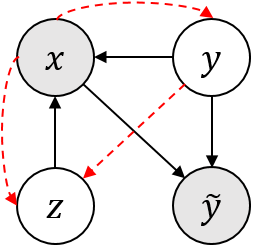}
            \caption{CausalNL}
            \label{fig:CausalNL}
        \end{subfigure}
        \begin{subfigure}[t]{.9\linewidth}
            \includegraphics[width=\textwidth]{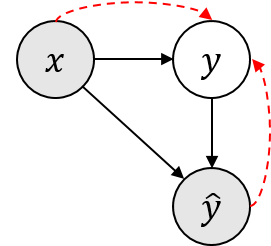}
            \caption{NPC (Ours)}
            \label{fig:mine_network}
        \end{subfigure}
    \end{minipage}
    \begin{minipage}{0.61\textwidth}
        \begin{subfigure}[t]{\linewidth}
            \includegraphics[width=\textwidth]{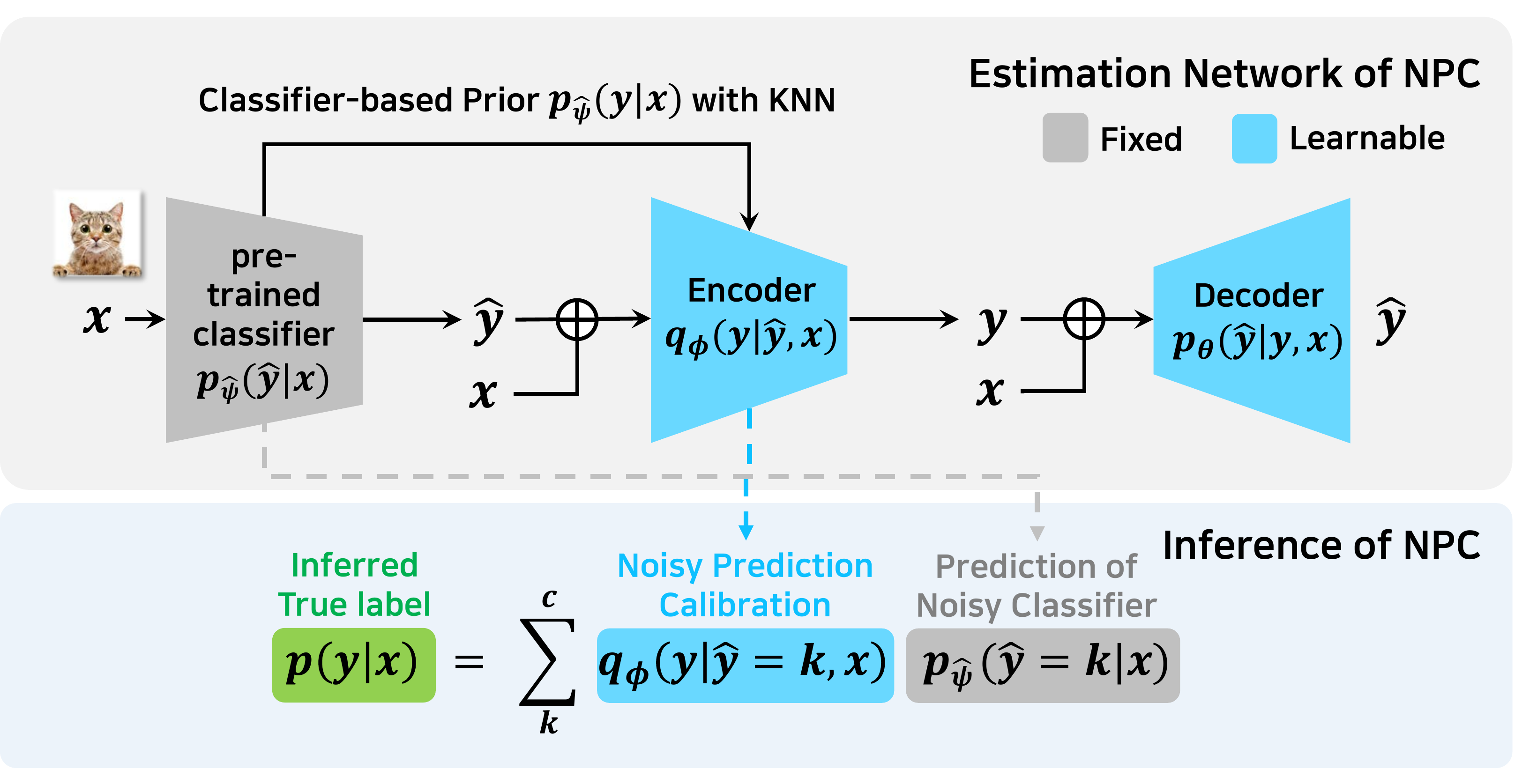}
            \caption{The Estimation Structure (upper) and Inference (under) of NPC}
            \label{fig:mine_NN}
        \end{subfigure}
    \end{minipage}
    \caption{\textbf{Structure.} Bayesian network of (a) CausalNL and (b) NPC. Arrows with solid lines and dashed lines denote generative process and inference process, respectively; illustrations of (c) Neural Network structure (upper) and inference (under) of NPC, respectively.}
\label{fig:method main figure}
\end{figure*}

\subsection{Motivational Comparison with Transition Matrix}
\label{sec:3.1}
Before we present our approach, we first explain our abstract view of noisy label classification in this section. Let $\psi$ be a set of model parameters of a classifier. Then, there are three cases of $\psi$: 1) $\tilde{\psi}$ which composes the function that explains the noisy labelled dataset perfectly; 2) $\hat{\psi}$ which is equivalent to classifiers trained with algorithms for learning with noisy labels; and 3) $\psi^{*}$ is from the optimal classifier that best explains the true joint distribution of $x$ and $y$.

Traditional methodologies for dealing with noisy labels have focused on finding $\psi^{*}$ with observation of ($x$,$\tilde{y}$). Without any access to the true joint distribution between $x$ and $y$, it results in $\hat{\psi}$. Corresponding to the definition of $\hat{\psi}$, a classifier $f$ maps $x$ to $\hat{y}$, which is a classifier prediction at the current training. Figure \ref{fig:method explanation} shows these mappings from $x$ to $y$, $\tilde{y}$, and $\hat{y}$.

Previously, there was no explicit modelling on $\hat{y}$ because $\hat{y}$ was considered as a transient state converging to $y$. This transience is modelled with $T$ in the previous research instead of articulating $\hat{y}$. $T$ is merged into the learning process as Eq. \ref{Eq:revised loss}, so $T$ becomes a force to converge on $y$ with only $\tilde{R}^{emp}_{L}$, which we represents as a purple texts in Figure \ref{fig:method explanation}.

We hypothesize that there could be also a chance of \textit{calibration} at the label space of $y$, instead of the parameter space of $\psi$. We call this procedure as Noisy Prediction Calibration (NPC) because $\hat{\psi}$ is assumed to be fixed and to produce $\hat{y}$. The calibration will transform $\hat{y}\rightarrow y$ by $H$ as a prediction adjustment, where our calibration matrix $H$ is formulated as Eq. \ref{eq:transition}. 
\begin{equation}
\begin{aligned}
\label{eq:transition}
p(y|x) &= \sum_{\Hat{y}}p(y|\Hat{y},x)p(\Hat{y}|x) \\
H_{kj}(x) &= p(y=j|\Hat{y}=k,x) \;\, for \; j,k=1,...,c
\end{aligned}
\end{equation}
The predicted label $\hat{y}$ from the classifier trained by algorithms for managing noisy labels has the potential to reduce the noise-level compared to the original noisy label $\tilde{y}$. Then, by explicitly mentioning $\hat{y}$ from the post-processing stage, we find a novel opportunity to adjust $\hat{y}$. If there is a trained classifier $f$ \footnote{Such black-box classifiers include deep neural networks with too many parameters \cite{gpt3,vit} for training given a user's computing environment.}, the previous methods with $T$ are not applicable to calibrate noisy labels to manage the relation from $\hat{y}$ to $y$. With post-processing mechanism, however, an adjustment $\hat{y}\rightarrow y$ without altering $\hat{\psi}$ can be applied to the pre-trained classifier as the red text in Figure \ref{fig:method explanation}.

The most critical difference between 1) a classifier training with noisy labels with $T$ and 2) a noisy prediction calibration with $H$ in post-processing is whether influencing on $\hat{\psi}$ or calibrating $\hat{y}$. Surely, the calibration quality of $H$ depends upon the quality of $\hat{y}$ and the structure of $\hat{H}$, the estimation of $H$. Section \ref{sec:model structure} describes our probabilistic estimation on $\hat{H}$ with minimal introduction of latent variables. Then, we show whether $T$ and $H$ can be analytically aligned with different formulation. Section \ref{sec: comp with trans} provides a theoretic claim that $T$ and $H$ are potentially interchangeable with some implementable weighting variables.

\subsection{NPC : Noisy Prediction Calibration Algorithm}
\label{sec:model structure}
As described in Eq. \ref{eq:transition}, we need to estimate $p(y|\hat{y},x)$ and $p(\Hat{y}|x)$. Since $p(\Hat{y}|x)$ can be obtained by training a classifier, we focus on the estimation of $p(y|\hat{y},x)$.

\subsubsection{Probabilistic Model Structure} 
We model a generative process of data instances with label noise as Figure \ref{fig:mine_network}. This generative process expresses the probabilistic relation of $y$ and $\tilde{y}$ given input $x$.
\begin{enumerate}
    \item Choose a latent vector $y$ $\sim$ \text{Dir}($\alpha_x$)
    \item Choose a noisy label $\tilde{y} \sim$ \text{Multi}($\pi_{x,y}$)
\end{enumerate}
$\alpha_{x}$ is the instance-dependent parameter of the prior probability distribution given a sample $x$, $\alpha_{x} \in \mathbb{R}_{+}^{c}$. Dir($\alpha_{x}$) is a Dirichlet distribution parameterized by $\alpha_{x}$, which is a conjugate prior of the corresponding multinomial distribution. $\pi_{x,y}^{k}$ is the probability for selecting a class $k=1,...,c$ for the noisy label, $\tilde{y}$ \footnote{$\pi_{x,y} \in \mathbb{R}_{+}^{c}$, $\sum_{k=1}^c \pi_{x,y}^k=1$}. Multi($\pi_{x,y}$) is the multinomial distribution with $\pi_{x,y}$.

This generative process describes the following scenario of noisy labelling. Given an input $x$, there is a prior distribution of its true label $y$ with a Dirichlet distribution. Then, the input content and its prior information on true label dictates a noisy label $\tilde{y}$. 

According to the generative process above, the joint probability $p(y,\Hat{y},x)$ can be factorized as follows:
\begin{equation}
p(y,\Hat{y},x)\propto p(y|x)p(\hat{y}|y,x)
\end{equation}
Here, as $p(y|x)$ is unknown from our training state, we approximate it with our pre-trained classifier parameterized by $\hat{\psi}$. The probabilities are defined as:
\begin{equation}
\begin{aligned}
    p_{\hat{\psi}}(y|x) = \text{Dir}(\alpha_x),\;p(\tilde{y}|y,x) = \text{Multi}(\pi_{x,y})
\label{eq:proba_def}
\end{aligned}
\end{equation}
\subsubsection{Parameter Inference} Our main estimation target is the posterior probability $p(y|\hat{y},x)$, which is intractable. Accordingly, we follow the variational inference \cite{vae} framework to minimize the KL divergence between the inference distribution and the target distribution. We introduce a variational distribution, $q(y|\hat{y},x)$, and we parametrize the variational and the remaining model distributions with $\phi$ and $\theta$, respectively. These formulations result in the KL divergence as Eq. \ref{eq:kl divergence}.

\begin{equation}
\begin{aligned}
\text{KL}&(q_{\phi}(y|\Hat{y},x)||p(y|\Hat{y},x)) \\
&= \int q_{\phi}(y|\Hat{y},x) \log \frac{q_{\phi}(y|\Hat{y},x)}{p(y|\Hat{y},x)}dy \\
&= \int q_{\phi}(y|\Hat{y},x) \log \frac{q_{\phi}(y|\Hat{y},x)p(\Hat{y}|x)}{p(\Hat{y}|y,x)p(y|x)}dy \\
&= \;\log p(\Hat{y}|x)-E_{y\sim q_{\phi}(y|\Hat{y},x)}\left[ \log p_{\theta}(\Hat{y}|y,x) \right] \qquad \\
&+\text{KL}(q_{\phi}(y|\Hat{y},x)||p(y|x)) = \log p(\Hat{y}|x)-\text{ELBO} \\
\end{aligned}
\label{eq:kl divergence}
\end{equation}

With a latent variable $y$ following the Dirichlet distribution, the direct reparameterization is difficult unlike a case with normal distribution prior \cite{vae}. Therefore, we utilize the reparametrization trick from \cite{dirichletvae}, which decomposes Dirichlet distribution into Gamma distributions with their inverse CDF\footnote{we describe the reparameterization details in Appendix \ref{appen: reparameterization}.}. Finally, Eq. \ref{eq:elbo} presents ELBO in Eq. \ref{eq:kl divergence} as a objective function.
\begin{equation}
\begin{aligned}
\text{ELBO} = &\sum_{k=1}^c\Hat{y}^i\log\Hat{y^*}^k+(1-\Hat{y}^k)\log(1-\Hat{y^*}^k) \\
& -\sum_{k=1}^c\log\Gamma(\alpha_{x}^k)+\sum_{k=1}^c\log\Gamma(\Hat{\alpha}_{x,\Bar{y}}^k)\\
& - \sum_{k=1}^c(\Hat{\alpha}_{x,\Bar{y}}^k-\alpha_{x}^k)\psi(\Hat{\alpha}_{x,\Bar{y}}^k)\\
\end{aligned}
\label{eq:elbo}
\end{equation}
Here, $\Hat{y^*}$ is the reconstruction output; $\Gamma$ and $\psi$ are the gamma and digamma function, respectively.

\subsubsection{Implementation}
\label{section:implementation}
Figure \ref{fig:mine_NN} shows the illustration of a neural network structure for NPC. Since we are interested in $p(y|\hat{y},x)$, we adopted the structure of Variational Autoencoder (VAE) to generate the parameter of variational posterior distribution.

We design a prior distribution of the latent variable $y$ to be dependent on $x$, which is enabled by utilizing the pre-trained classifier $\hat{\psi}$. We apply K-Nearest Neighbor (KNN) algorithm \cite{knn1, knn2} to samples with high confidence in the classifier output, $p_{\hat{\psi}}(\hat{y}|x)$. The most selected label from k neighbours is referred as $\Bar{y}$, and we differentiate the parameter values of the prior Dirichlet distribution of $y$ by $\Bar{y}$ as Eq. \ref{eq: prior}.
\begin{equation}
\begin{aligned}
\alpha_x^k& = \left\{\begin{matrix}
\delta & k\neq\Bar{y}\\ 
\delta+\rho & k=\Bar{y}
\end{matrix}\right. \; for\; k=1,...,c
\end{aligned}
\label{eq: prior}
\end{equation}
Here, $\delta$ and $\rho$ are hyper-parameters to setup $\alpha$ of the Dirichlet distribution. Throughout the paper, we fixed $\delta=1$.

While setting the prior parameters as above, we obtain the posterior distribution from the encoder. We select the soft plus function as the activation function of the encoder to make the posterior distribution parameter non-negative. The resulting posterior distribution provides a parameter sample of the posterior multinomial distribution $H$. We set the parameter of $H$ to be the mode of the posterior Dirichlet distribution, which becomes $H=p(y|\Hat{y},x)=\frac{\alpha_{x,\Hat{y}}^y-1}{\sum_{y=1}^c\alpha_{x,\Hat{y}}^y-c}$.

With the definition of $H$, we calibrate the noisy prediction described as in Figure \ref{fig:method explanation}, indicating $\hat{y}\rightarrow y$. This is an transition from the prediction of noisy classifier $(p(\hat{y}|x))$ to the true probability of label distribution $(p(y|x))$, which we will calculate as Eq. \ref{eq: prior_vae}.
\begin{equation}
\begin{aligned}
p(y|x)&=\sum_{k}p(y|\hat{y}=k,x)p(\hat{y}=k|x)\\&\approx\sum_{k}q_{\phi}(y|\hat{y}=k,x)p(\hat{y}=k|x)
\label{eq: prior_vae}
\end{aligned}
\end{equation}
\subsection{Alignment of $T$ and $H$}
\label{sec: comp with trans}
When we utilize NPC as a post-processing algorithm, the classifier $p_{\hat{\psi}}(\hat{y}|x)$ becomes a fixed model in Figure \ref{fig:mine_NN}. Therefore, training $H$ does not affect $\psi$, the classifier parameters, unlike training $T$. 
Nevertheless, we can relate our modeling with $T$ as stated in Proposition \ref{pro:transition relation} (full proof in Appendix \ref{appen: proposition proof}). This alignment shows that our methodology provides a same pathway to correct the noisy classifier as in the previous studies on $T$.
\begin{proposition}
Assume that $\Hat{y}$ and $y$ are conditionally independent given $\tilde{y}$, and $p(\Hat{y}=k|x) \neq 0$ for all $k=1,..,c$. Then, $H_{kj}(x) = \frac{p(y=j|x)}{p(\widehat{y}=k|x)}\sum_{i}p(\widehat{y}=k|\widetilde{y}=i,x)T_{ij}(x)$ for all $j,k=1,...,c$.
\label{pro:transition relation}
\end{proposition}
The assumption of $\hat{y} \indep y|\tilde{y}$ is natural because $\hat{y}$ is conditionally independent to $y$ when $\hat{\psi}$ is trained only with $\tilde{y}$. Proposition \ref{pro:transition relation} shows that $H$ can be formulated by $T$ with weighting variables, which are observable or can be easily computed from our framework. It implies that we can also infer $T$ from the inference procedure of $H$, which suggests the possibility of relating the framework of NPC to the transition matrix based approaches.
\begin{figure}[h]
    \centering
    \begin{minipage}{0.325\columnwidth}
        \begin{subfigure}[t]{\linewidth}
            \includegraphics[width=\textwidth]{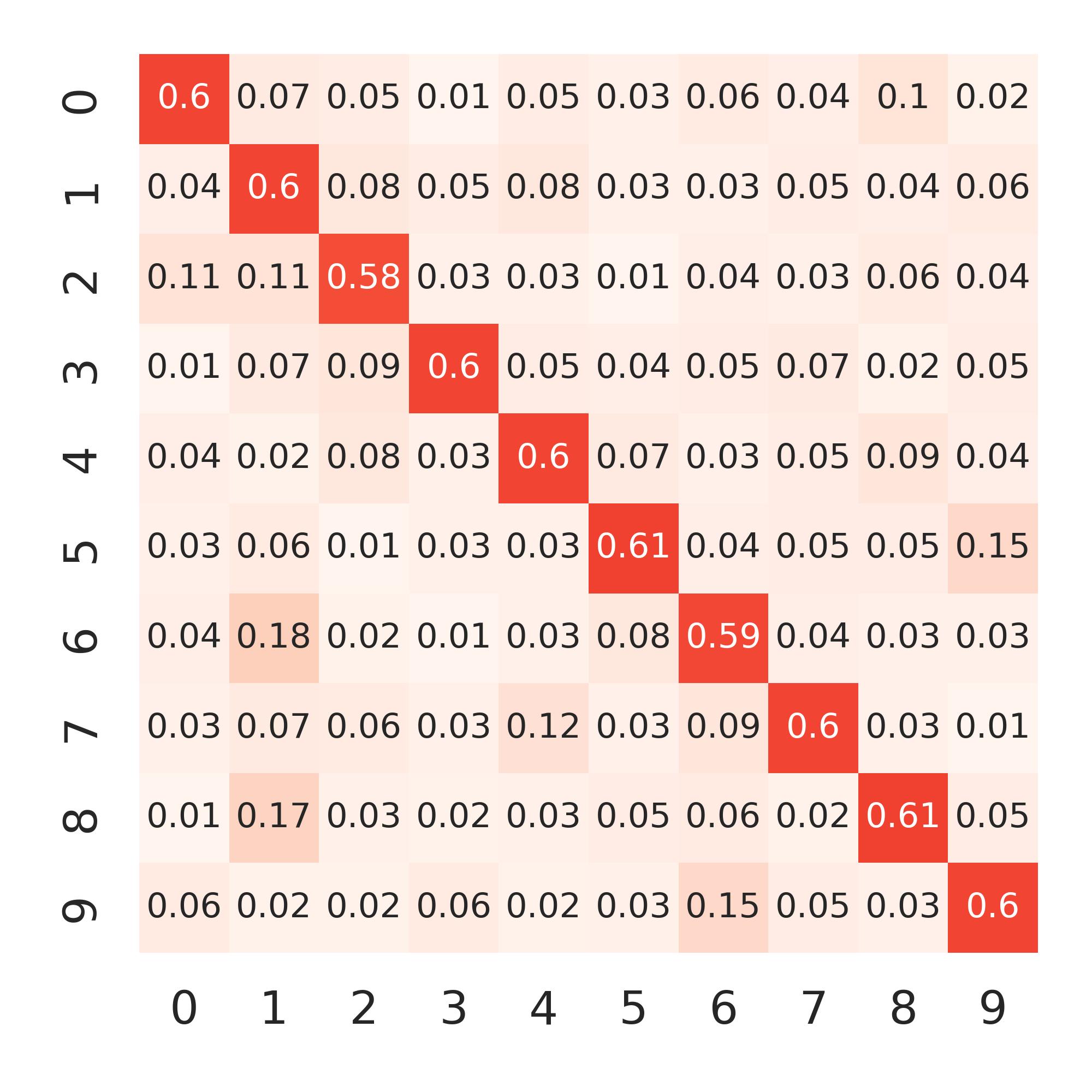}
            \caption{True Transition}
            \label{fig:true_transition}
        \end{subfigure}
    \end{minipage}
    \begin{minipage}{0.325\columnwidth}
        \begin{subfigure}[t]{\linewidth}
            \includegraphics[width=\textwidth]{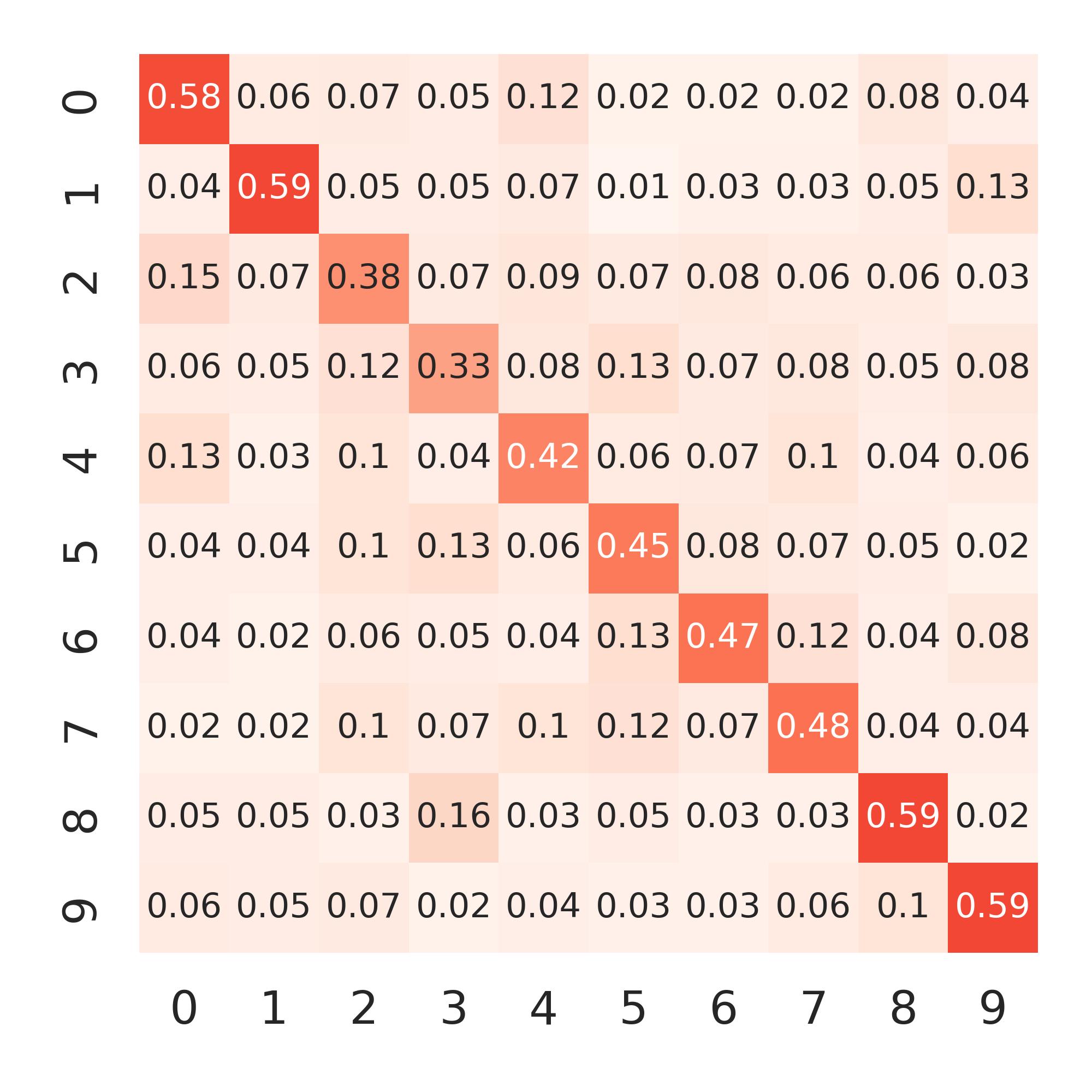}
            \caption{Forward (0.004)}
            \label{fig:forward_transition}
        \end{subfigure}
    \end{minipage}
    \begin{minipage}{0.325\columnwidth}
        \begin{subfigure}[t]{\linewidth}
            \includegraphics[width=\textwidth]{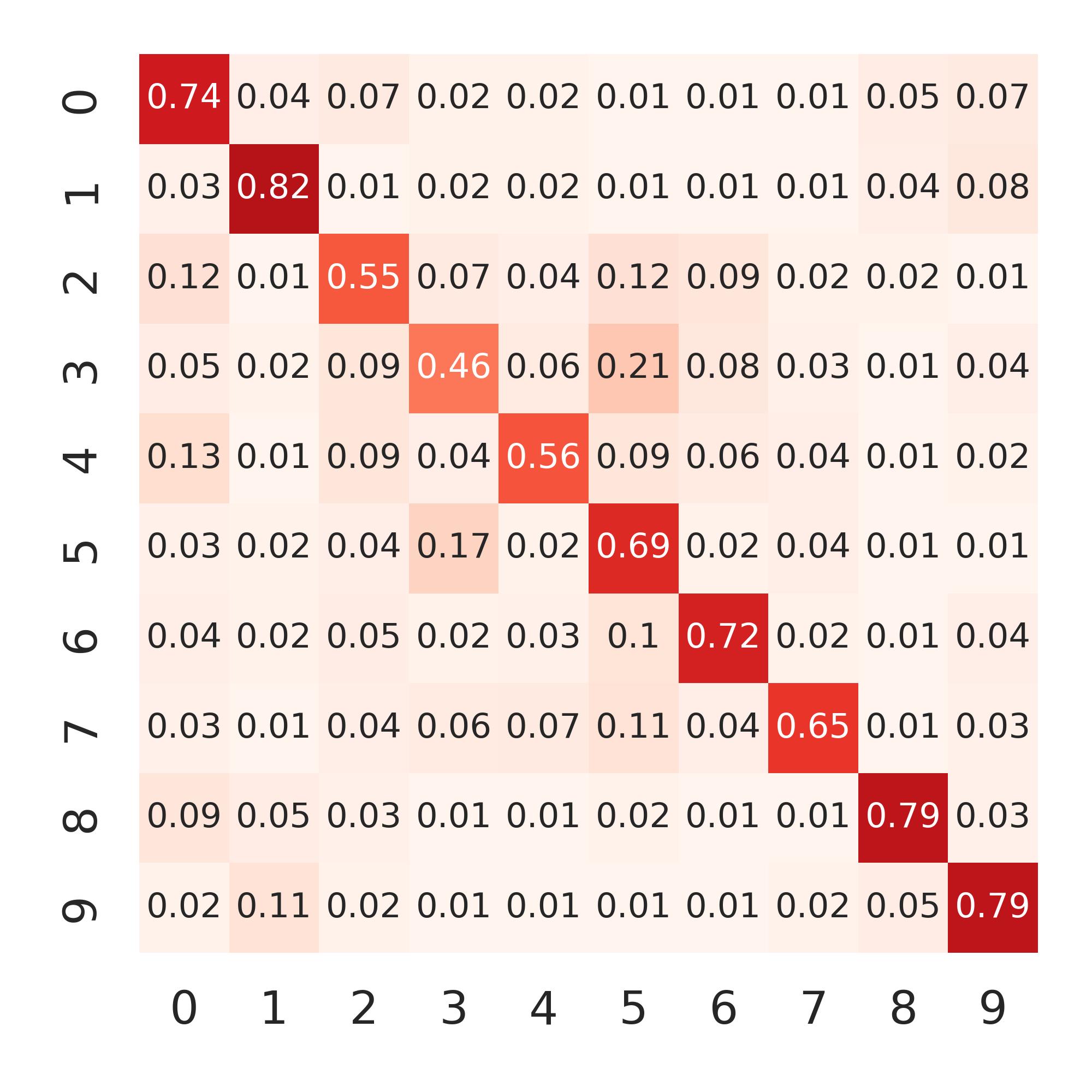}
            \caption{Dualt (0.004)}
            \label{fig:dualt_transition}
        \end{subfigure}
    \end{minipage}\par
    \begin{minipage}{0.325\columnwidth}
        \begin{subfigure}[t]{\linewidth}
            \includegraphics[width=\textwidth]{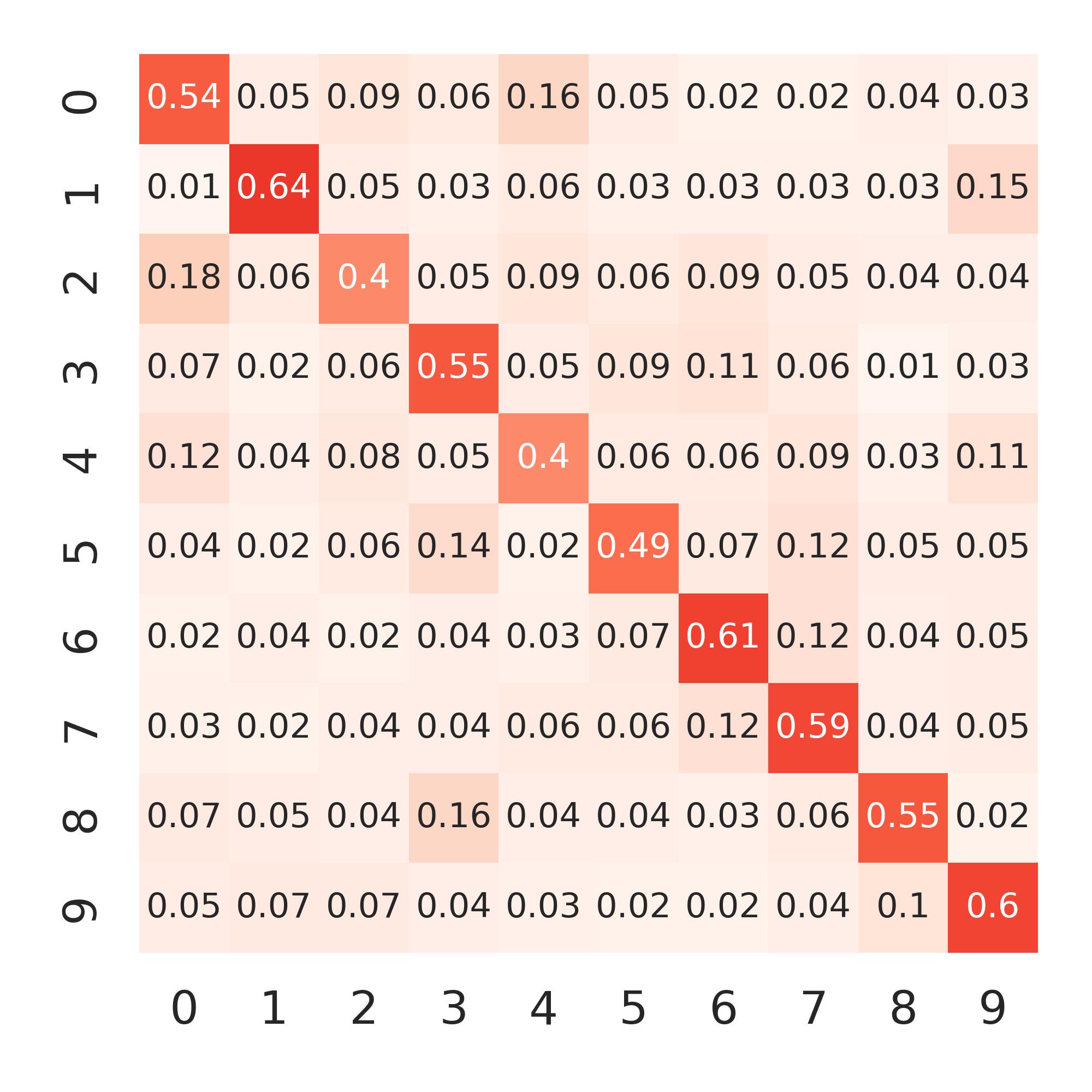}
            \caption{TVR (0.003)}
            \label{fig:tvr_transition}
        \end{subfigure}
    \end{minipage}
    \begin{minipage}{0.325\columnwidth}
        \begin{subfigure}[t]{\linewidth}
            \includegraphics[width=\textwidth]{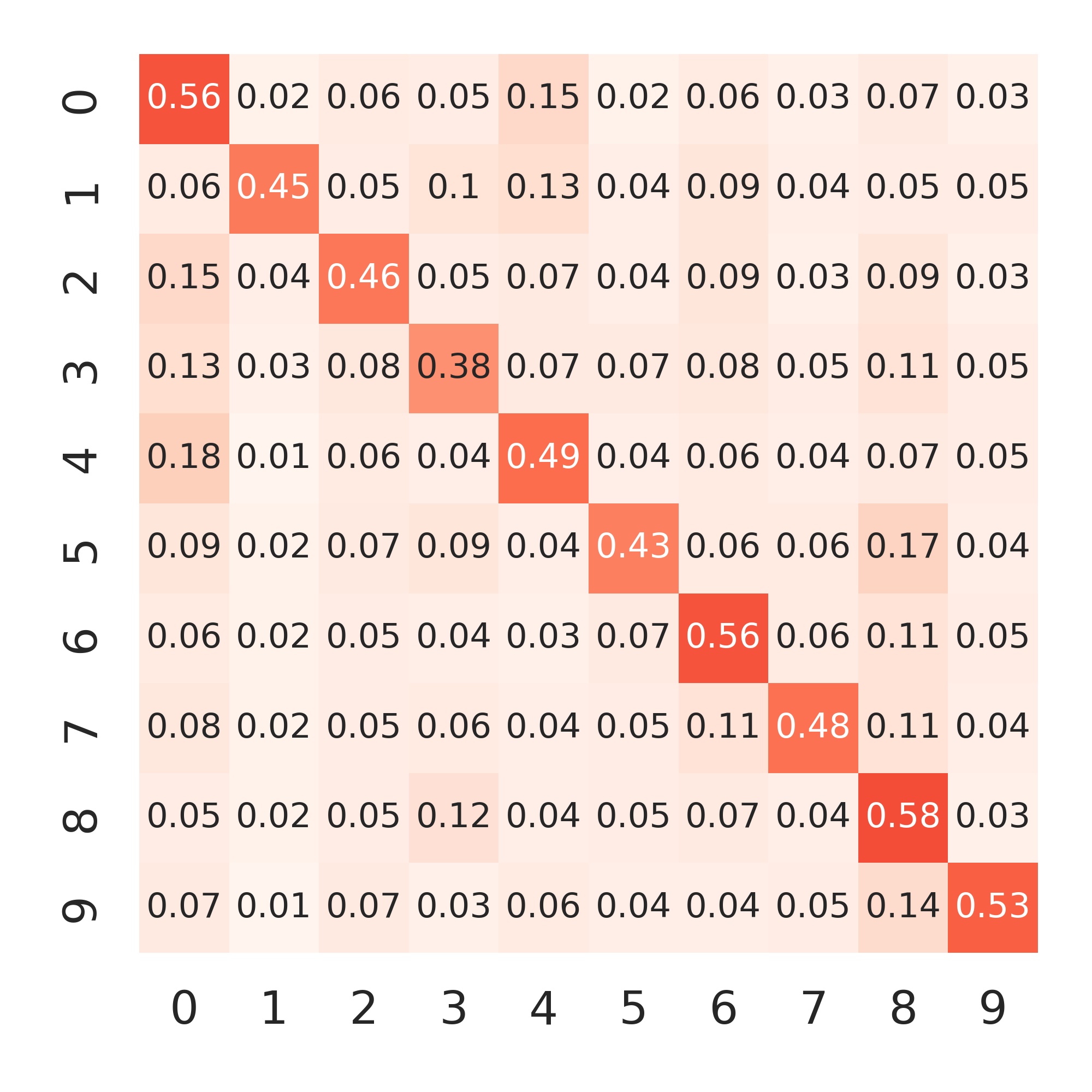}
            \caption{CausalNL (0.005)}
            \label{fig:causal_transition}
        \end{subfigure}
    \end{minipage}
    \begin{minipage}{0.325\columnwidth}
        \begin{subfigure}[t]{\linewidth}
            \includegraphics[width=\textwidth]{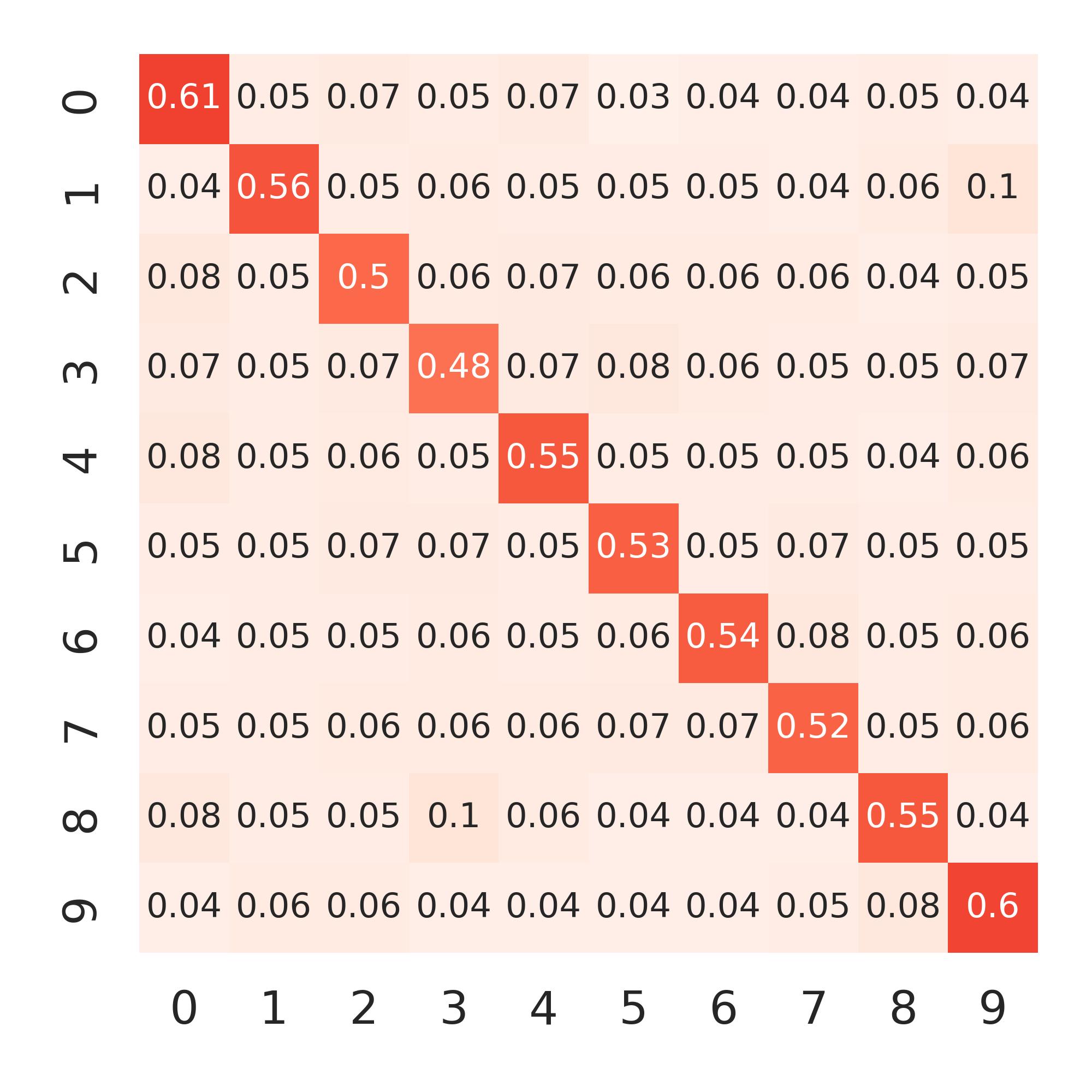}
            \caption{NPC (0.002)}
            \label{fig:npc_transition}
        \end{subfigure}
    \end{minipage}
    \caption{(a) Average of the true transition matrix T of CIFAR-10 with instance dependent noise 40\%, (b)-(e) the estimated transition matrix of each algorithm. Values in parentheses are the mean squared error between estimated $\hat{T}$ and true $T$.}
\label{fig:transition_matrix_comparison}
\end{figure}

Figure \ref{fig:transition_matrix_comparison} shows (a) true transition matrix $T$ and the empirically estimated $\hat{T}$ (b)-(e) from various algorithms for CIFAR-10 dataset with 40$\%$ instance dependent noise. We calculate the mean squared error between true $T$ and estimated $\hat{T}$ from each method. It should be noted that the estimated $\hat{T}$ is an aggregated matrix from the input-wisely estimated matrix. Although the main inference target of NPC is not $T$ but $H$, NPC shows lowest estimation error of $T$, even without any access to the noise ratio. Details of the implemented experiment and more results of transition matrix estimation on other noise conditions are provided in Appendix \ref{appen: t and h}.

\section{Experiments}
\label{sec:exper}
In this section, we first introduce implementation details and baselines for experiments in \ref{sec:datasets and baselines}. Afterwards, we provide quantitative performance results in \ref{sec:quanti} and qualitative analyses in \ref{sec:quali}. We also compare our method with other post-processing methods in \ref{sec:post}, with CausalNL in \ref{sec:causalNL}. Finally, we provide NPC's ability as a detector of potentially noisy samples in the benchmark dataset in \ref{sec:realdata}.

\subsection{Datasets and Baselines}
\label{sec:datasets and baselines}
To verify the efficacy of our proposed method, NPC, we first experiment on three benchmark datasets: \textit{MNIST} \cite{mnist}, \textit{Fashion-MNIST} \cite{fmnist}, and \textit{CIFAR-10} \cite{cifar10}. We inject synthetic label noises following the previous setups of existing works \cite{forward, coteaching, rel, csidn, causalNL}. Four different types of artificial label noise are generated as below. More details on each noise type are provided in Appendix \ref{appen : noisy label generation}. We also consider the clean setting without any noise to analyze the generalization ability of NPC.
\begin{itemize}
    \item \textbf{Symmetric Noise (SN)} flips labels uniformly to other classes \cite{forward, coteaching, rel}.
    \item \textbf{Asymmetric Noise (ASN)} flips labels within similar classes \cite{joint, coteaching, rel}.
    \item \textbf{Instance Dependent Noise (IDN)} flips labels by the probability which is proportional to the feature of a sample \cite{pdn, cores, causalNL}.
    \item \textbf{Similarity Reflected Instance Dependent Noise (SRIDN)} not only considers instance-feature, but also takes into account of class similarity \cite{aaai-idn, csidn}.
\end{itemize}
Also, we compute experiments on two real-world datasets: \textit{Clothing-1M} \cite{clothing1m} and \textit{Food101} \cite{food}. Details of datasets and implementations are provided in Appendix \ref{appen:real_data} and \ref{appen:experiment condition}, respectively.

We demonstrate the performance of NPC as a post-processor, by applying NPC to a classifier pre-trained with baseline methods. The selected algorithms include: (1) \textbf{CE} (Cross Entropy), (2) \textbf{Joint} \cite{joint}, (3) \textbf{Coteaching} \cite{coteaching}, (4) \textbf{JoCoR} \cite{jocor}, (5) \textbf{CORES2} \cite{cores}, (6) \textbf{SCE} \cite{sce}, (7) \textbf{ES} (EarlyStop), (8) \textbf{LS} (Label Smoothing) \cite{ls}, (9) \textbf{REL} \cite{rel}, (10) \textbf{Forward} \cite{forward}, (11) \textbf{DualT} \cite{dualT}, (12) \textbf{TVR} \cite{tvr}, and (13) \textbf{CausalNL} \cite{causalNL}. Characteristics of each classifier are explained in Appendix \ref{baselines}.

\definecolor{Gray}{gray}{0.9}
\begin{table*}[t]
\centering
\resizebox{0.99\textwidth}{!}{%
    \begin{tabular}{c cc ccccccccc ccccccccc}
    \toprule
     & \multicolumn{2}{c }{\textsf{MNIST}} & \multicolumn{9}{c}{\textsf{Fashion-MNIST}} & \multicolumn{9}{c}{\textsf{CIFAR-10}} \\ \cmidrule(lr){2-3}\cmidrule(lr){4-12}\cmidrule(lr){13-21} \multicolumn{1}{c}{\textbf{Model}}
     & \textbf{Clean} & \multicolumn{1}{c}{\textbf{IDN}} & \textbf{Clean} & \multicolumn{2}{c}{\textbf{SN}} & \multicolumn{2}{c}{\textbf{ASN}} & \multicolumn{2}{c}{\textbf{IDN}} & \multicolumn{2}{c}{\textbf{SRIDN}} & \multicolumn{1}{c}{\textbf{Clean}}& \multicolumn{2}{c}{\textbf{SN}} & \multicolumn{2}{c}{\textbf{ASN}} & \multicolumn{2}{c}{\textbf{IDN}} & \multicolumn{2}{c}{\textbf{SRIDN}} \\
     \cmidrule(lr){2-2}\cmidrule(lr){3-3}\cmidrule(lr){4-4}\cmidrule(lr){5-6}\cmidrule(lr){7-8}\cmidrule(lr){9-10}\cmidrule(lr){11-12}\cmidrule(lr){13-13}\cmidrule(lr){14-15}\cmidrule(lr){16-17}\cmidrule(lr){18-19}\cmidrule(lr){20-21}
      & \multicolumn{1}{c}{-} & \multicolumn{1}{c}{40$\%$}& \multicolumn{1}{c}{-} & 20$\%$ & \multicolumn{1}{c}{80$\%$} & 20$\%$ & \multicolumn{1}{c}{40$\%$} & 20$\%$ & \multicolumn{1}{c}{40$\%$} & 20$\%$ & \multicolumn{1}{c}{40$\%$} & \multicolumn{1}{c}{-} & 20$\%$ & \multicolumn{1}{c}{80$\%$} & 20$\%$  & \multicolumn{1}{c}{40$\%$} & 20$\%$ & \multicolumn{1}{c}{40$\%$} & 20$\%$ & \multicolumn{1}{c}{40$\%$} \\ \midrule 
    CE & 97.8 & 66.3 & 87.1 & 74.0 & 27.0 & 81.0 & 77.3 & 68.4 & 52.1 & 81.0 & 67.3 & 86.9 & 73.1 & 15.1 & 80.2 & 71.4 & 72.9 & 53.9 & 72.6 & 61.8 \\
    \rowcolor{Gray} w/ NPC& \textbf{98.2} & \textbf{89.0} & \textbf{88.4} & \textbf{84.0} & \textbf{35.8} & \textbf{85.9} & \textbf{86.2} & \textbf{82.5} & \textbf{74.5} & \textbf{81.8} & \textbf{69.4} & \textbf{89.0} & \textbf{80.8} & \textbf{17.0} & \textbf{84.7} & \textbf{78.8} & \textbf{80.9} & \textbf{59.9} & \textbf{74.3} & \textbf{64.3} \\ \midrule
    Joint & 93.0 & 93.6 & 82.8 & 82.0 & \textbf{6.0} & 82.1 & 82.3 & 82.7 & 82.4 & 80.6 & 74.6 & 83.0 & 78.9 & \textbf{8.3} & 81.5 & 76.8 & 80.4 & 64.5 & 70.6 & 62.2 \\
    \rowcolor{Gray}w/ NPC& \textbf{94.0} & \textbf{94.6} & \textbf{83.6} & \textbf{82.7} & \textbf{6.0} & \textbf{82.9} & \textbf{82.9} & \textbf{83.4} & \textbf{83.0} & \textbf{81.1} & \textbf{75.5} & \textbf{84.4} & \textbf{80.2} & \textbf{8.3} & \textbf{83.0} & \textbf{77.7} & \textbf{80.7} & \textbf{69.1} & \textbf{72.0} & \textbf{63.6} \\ \midrule
    Coteaching & 98.0 & 87.5 & 87.0 & 82.5 & 64.2 & 88.2 & \textbf{73.6} & 81.8 & 75.4 & 84.0 & 75.0 & 88.5 & 82.5 & 29.7 & 86.5 & 76.6 & 81.5 & 75.2 & 75.3 & 66.6 \\
    \rowcolor{Gray}w/ NPC & \textbf{98.3} & \textbf{90.6} & \textbf{88.3} & \textbf{85.8} & \textbf{66.0} & \textbf{88.5} & \textbf{73.6} & \textbf{85.1} & \textbf{78.7} & \textbf{84.2} & \textbf{75.3} & \textbf{89.2} & \textbf{85.3} & \textbf{32.1} & \textbf{87.1} & \textbf{76.8} & \textbf{84.8} & \textbf{78.5} & \textbf{76.1} & \textbf{67.2} \\ \midrule
    JoCoR & 97.8 & 93.3 & 88.7 & 86.0 & 27.6 & 88.9 & 79.4 & 86.3 & 83.2 & 81.9 & 71.3 & 89.1 & 83.6 & 24.8 & 82.6 & 73.3 & 82.8 & 75.3 & 75.2 & 66.1 \\
    \rowcolor{Gray}w/ NPC & \textbf{98.3} & \textbf{96.1} & \textbf{89.8} & \textbf{88.0} & \textbf{31.5} & \textbf{89.2} & \textbf{82.7} & \textbf{88.0} & \textbf{85.7} & \textbf{82.2} & \textbf{72.3} & \textbf{89.3} & \textbf{86.0} & \textbf{27.0} & \textbf{85.1} & \textbf{79.0} & \textbf{85.8} & \textbf{80.1} & \textbf{75.9} & \textbf{66.7} \\ \midrule
    CORES2 & 97.0 & 48.8 & 87.2 & 74.6 & 8.9 & 77.6 & 74.3 & 80.0 & 58.1 & 81.3 & 71.2 & 87.1 & 70.1 & \textbf{31.2} & 79.0 & 71.2 & 70.3 & 50.9 & 72.8 & 62.0 \\
    \rowcolor{Gray}w/ NPC & \textbf{98.0} & \textbf{67.2} & \textbf{88.5} & \textbf{84.3} & \textbf{10.2} & \textbf{82.5} & \textbf{81.0} & \textbf{84.0} & \textbf{69.6} & \textbf{82.2} & \textbf{74.9} & \textbf{88.2} & \textbf{80.4} & 30.7 & \textbf{84.2} & \textbf{80.4} & \textbf{80.4} & \textbf{65.6} & \textbf{74.2} & \textbf{64.1} \\ \midrule
    SCE  & 97.7 & 66.6 & 87.0 & 74.0 & 27.0 & 82.0 & 77.4 & 68.3 & 52.0 & 81.1 & 67.5 & 86.9 & 73.1 & 15.1 & 80.2 & 71.4 & 72.9 & 53.9 & 72.6 & 61.8 \\
    \rowcolor{Gray}w/ NPC & \textbf{98.2} & \textbf{88.7} & \textbf{88.3} & \textbf{83.7} & \textbf{35.5} & \textbf{86.4} & \textbf{86.7} & \textbf{82.0} & \textbf{75.2} & \textbf{81.8} & \textbf{69.7} & \textbf{87.4} & \textbf{75.0} & \textbf{15.2} & \textbf{81.5} & \textbf{75.2} & \textbf{75.4} & \textbf{55.6} & \textbf{72.9} & \textbf{62.5} \\ \midrule
    Early Stop & 96.5 & 73.3 & 87.5 & 83.6 & 49.5 & 84.1 & 76.6 & 79.5 & 55.4 & 83.3 & 72.6 & 83.0 & 79.1 & 18.0 & 80.9 & 70.6 & 77.1 & 62.5 & 71.4 & 60.6 \\
    \rowcolor{Gray}w/ NPC & \textbf{97.9} & \textbf{90.8} & \textbf{88.7} & \textbf{85.9} & \textbf{62.9} & \textbf{87.6} & \textbf{87.1} & \textbf{84.3} & \textbf{75.3} & \textbf{84.0} & \textbf{76.0} & \textbf{84.0} & \textbf{82.5} & \textbf{18.2} & \textbf{81.2} & \textbf{72.0} & \textbf{79.4} & \textbf{65.1} & \textbf{72.1} & \textbf{63.0} \\ \midrule
    LS  & 97.8 & 66.2 & 87.5 & 73.9 & 27.8 & 81.5 & 77.0 & 69.0 & 52.5 & 81.1 & 67.5 & 86.9 & 73.1 & 15.1 & 80.2 & 71.4 & 72.9 & 53.9 & 72.6 & 61.8 \\
    \rowcolor{Gray}w/ NPC & \textbf{98.2} & \textbf{88.6} & \textbf{88.6} & \textbf{83.7} & \textbf{35.2} & \textbf{86.0} & \textbf{86.4} & \textbf{82.2} & \textbf{74.7} & \textbf{81.6} & \textbf{69.5} & \textbf{89.0} & \textbf{80.8} & \textbf{15.5} & \textbf{84.7} & \textbf{78.8} & \textbf{80.9} & \textbf{59.9} & \textbf{74.3} & \textbf{64.3} \\ \midrule
    REL & \textbf{98.0} & 90.7 & \textbf{88.1} & 84.6 & 70.1 & 82.8 & 76.2 & \textbf{84.6} & 75.5 & \textbf{83.7} & 78.1 & 80.7 & 74.9 & 21.2 & 72.8 & 69.9 & 75.5 & \textbf{51.8} & 69.3 & 63.8 \\
    \rowcolor{Gray}w/ NPC & 97.9 & \textbf{95.5} & 86.9 & \textbf{85.0} & \textbf{70.3} & \textbf{85.3} & \textbf{83.0} & 83.8 & \textbf{80.1} & 82.9 & \textbf{78.3} & \textbf{83.4} & \textbf{78.6} & \textbf{26.0} & \textbf{75.9} & \textbf{76.1} & \textbf{78.5} & 51.2 & \textbf{70.7} & \textbf{64.2} \\ \midrule
    Forward & 98.0 & 67.9 & 88.5 & 77.4 & 24.3 & 83.3 & 79.2 & 75.2 & 56.9 & 82.4 & 69.5 & 85.3 & 71.8 & 16.9 & 78.2 & 70.1 & 70.2 & 54.5 & 73.2 & 63.5 \\
    \rowcolor{Gray}w/ NPC & \textbf{98.4} & \textbf{91.1} & \textbf{89.6} & \textbf{85.3} & \textbf{33.0} & \textbf{87.2} & \textbf{86.8} & \textbf{86.8} & \textbf{80.5} & \textbf{83.3} & \textbf{73.7} & \textbf{88.7} & \textbf{81.5} & \textbf{17.2} & \textbf{83.8} & \textbf{74.5} & \textbf{80.3} & \textbf{63.3} & \textbf{74.8} & \textbf{65.0} \\ \midrule
    DualT & 96.7 & 94.3 & 86.3 & 84.5 & \textbf{10.0} & 86.9 & 83.1 & 85.1 & 68.5 & 82.7 & 73.2 & 84.3 & 79.3 & 7.6 & 80.6 & 77.1 & 78.6 & 71.2 & 68.7 & 63.1 \\
    \rowcolor{Gray}w/ NPC & \textbf{97.8} & \textbf{96.6} & \textbf{88.2} & \textbf{85.9} & \textbf{10.0} & \textbf{87.6} & \textbf{84.3} & \textbf{86.3} & \textbf{72.3} & \textbf{83.4} & \textbf{74.9} & \textbf{86.0} & \textbf{83.0} & \textbf{8.4} & \textbf{83.0} & \textbf{77.5} & \textbf{81.0} & \textbf{77.3} & \textbf{70.1} & \textbf{64.0} \\ \midrule
    TVR & 97.7 & 64.4 & 87.0 & 72.6 & 24.9 & 80.6 & 76.4 & 66.3 & 51.7 & 81.4 & 67.7 & 86.7 & 71.9 & 15.2 & 78.5 & 71.2 & 72.3 & 53.6 & 72.2 & 62.2 \\
    \rowcolor{Gray}w/ NPC & \textbf{98.1} & \textbf{84.5} & \textbf{88.3} & \textbf{82.3} & \textbf{31.9} & \textbf{84.9} & \textbf{85.3} & \textbf{79.8} & \textbf{73.6} & \textbf{82.1} & \textbf{70.3} & \textbf{88.3} & \textbf{80.8} & \textbf{15.7} & \textbf{84.1} & \textbf{76.5} & \textbf{80.8} & \textbf{60.7} & \textbf{74.5} & \textbf{64.5} \\ \midrule
    CausalNL & 98.1 & 85.2 & 88.1 & 84.0 & 51.5 & 88.8 & 87.4 & 83.4 & 75.2 & 82.0 & 71.2 & 89.6 & 79.9 & 17.0 & 84.6 & 74.8 & 79.9 & 60.4 & 74.6 & 63.5 \\
    \rowcolor{Gray}w/ NPC & \textbf{98.6} & \textbf{94.5} & \textbf{89.4} & \textbf{87.0} & \textbf{58.9} & \textbf{89.3} & \textbf{88.7} & \textbf{87.6} & \textbf{83.3} & \textbf{83.3} & \textbf{74.1} & \textbf{89.7} & \textbf{81.2} & \textbf{18.8} & \textbf{85.0} & \textbf{74.8} & \textbf{81.2} & \textbf{71.9} & \textbf{75.3} & \textbf{63.9} \\ \bottomrule
    \end{tabular}%
}
\caption{Test accuracies for MNIST, Fashion MNIST and CIFAR-10 datasets with their labels corrupted by four types of noisy label conditions. We demonstrate averaged performances computed by baselines and the post-hoc performances after applying NPC. The experimental results are averaged values over five trials. \textbf{Bolded} text denotes the one with better performance. For MNIST, experiment results on other noise conditions are provided in Appendix \ref{appen:MNIST result}.}
\label{tab:acc_synthetic}
\end{table*}

\subsection{Classification Performance on Noisy Datasets}
\label{sec:quanti}

Table \ref{tab:acc_synthetic} shows the experimental results on three synthetic datasets with the four types of noisy types of various noisy ratios. By applying NPC to the selected 13 baselines as a post-processing model, we have total of 351 experimental cells including the clean cases. With five replications with different seeds, we got 341 cells with statistically significant improvement by NPC.

This result demonstrates that NPC is widely applicable to any type of pre-trained classifiers; and that NPC effectively calibrates any types of noisy labels without any access to noise information, i.e. noise ratio and noise type. It is noteworthy that NPC achieves impressive performances on IDN conditions on average, and we conjecture that with its generative modelling structure, NPC successfully calibrates instance dependent noise, which is quite a realistic noisy label generation process. Moreover, we observe that there are consistent performance enhancements in the clean cases. These clean-case improvements state that NPC also improves the generalization of the applied classifier.

\begin{table}[h]
    \resizebox{0.95\columnwidth}{!}{%
    \centering
    \begin{tabular}{c cccc}
        \toprule  & \multicolumn{2}{c}{\textsf{Food-101}}&\multicolumn{2}{c}{\textsf{Clothing1M}} \\ \cmidrule(lr){2-3}\cmidrule(lr){4-5}
        Method & w.o/ NPC & w/ NPC & w.o/ NPC & w/ NPC\\ \midrule
        CE & 78.37& \textbf{80.21}$\pm$\scriptsize{0.2}&68.14 & \textbf{70.83}$\pm$\scriptsize{0.1}\\
        Early Stop  & 73.22 & \textbf{76.80}$\pm$\scriptsize{0.3} & 67.07 & \textbf{70.21}$\pm$\scriptsize{0.1}\\
        SCE &75.23 & \textbf{78.26}$\pm$\scriptsize{0.3}&67.77 & \textbf{70.36}$\pm$\scriptsize{0.1}\\
        REL & 78.96& 78.95$\pm$\scriptsize{0.4} & 62.53& \textbf{64.83}$\pm$\scriptsize{0.1}\\
        Forward & 83.76& 83.77$\pm$\scriptsize{0.3}& 66.86 & \textbf{70.02}$\pm$\scriptsize{0.1}\\
        DualT &57.46 &\textbf{61.82}$\pm$\scriptsize{0.7} & 70.18& 69.99$\pm$\scriptsize{0.4}\\
        TVR & 77.34& \textbf{79.37}$\pm$\scriptsize{0.1}&67.18 & \textbf{69.44}$\pm$\scriptsize{0.1}\\
        CausalNL & 86.08 & \textbf{86.29}$\pm$\scriptsize{0.0} & 68.31 & \textbf{69.90}$\pm$\scriptsize{0.2}\\ \bottomrule
    \end{tabular}
    }
    \caption{Classification accuracy on \textit{Food-101} and \textit{Clothing1M}. w.o. means without and w. means with. Experiments are replicated over 5 times. \textbf{Bolded} text denotes improved performance with statistical significance.}
    \label{tab:comparison_real}
\end{table}

\noindent
While the previous results originate from synthetic datasets, Table \ref{tab:comparison_real} shows the experimental results on \textit{Food-101} and \textit{Clothing1M}, which are datasets with real world noisy labels with its test dataset containing human-annotated labels. Again, NPC improves the performance of baseline classifiers on their classification accuracy with statistical significance in 13 out of 16 experimental cells.

\subsection{Qualitative Analysis of $q_{\phi}(y|\hat{y},x)$}
\label{sec:quali}
Figure \ref{fig:tsne_comparison} shows t-SNE mapping views and confusion matrices of our samples from variational posterior, $q_{\phi}(y|\hat{y},x)$\footnote{The analysis is based on MNIST dataset with 40\% IDN.}. Each dot is colored by the original classifier output $\hat{y}$ (Figure \ref{fig:tsne_upper}) or $y^{*}=argmax_{y}P(y|x)$ calibrated with NPC (Figure \ref{fig:tsne_lower}), respectively. Comparing two t-SNE views, we can observe that the latent variable $y$ is well clustered to reflect the class information of a sample.

\begin{figure}[h]
\vskip-0.1in
    \centering
    \begin{subfigure}[t]{0.242\columnwidth}
        \includegraphics[width=\textwidth]{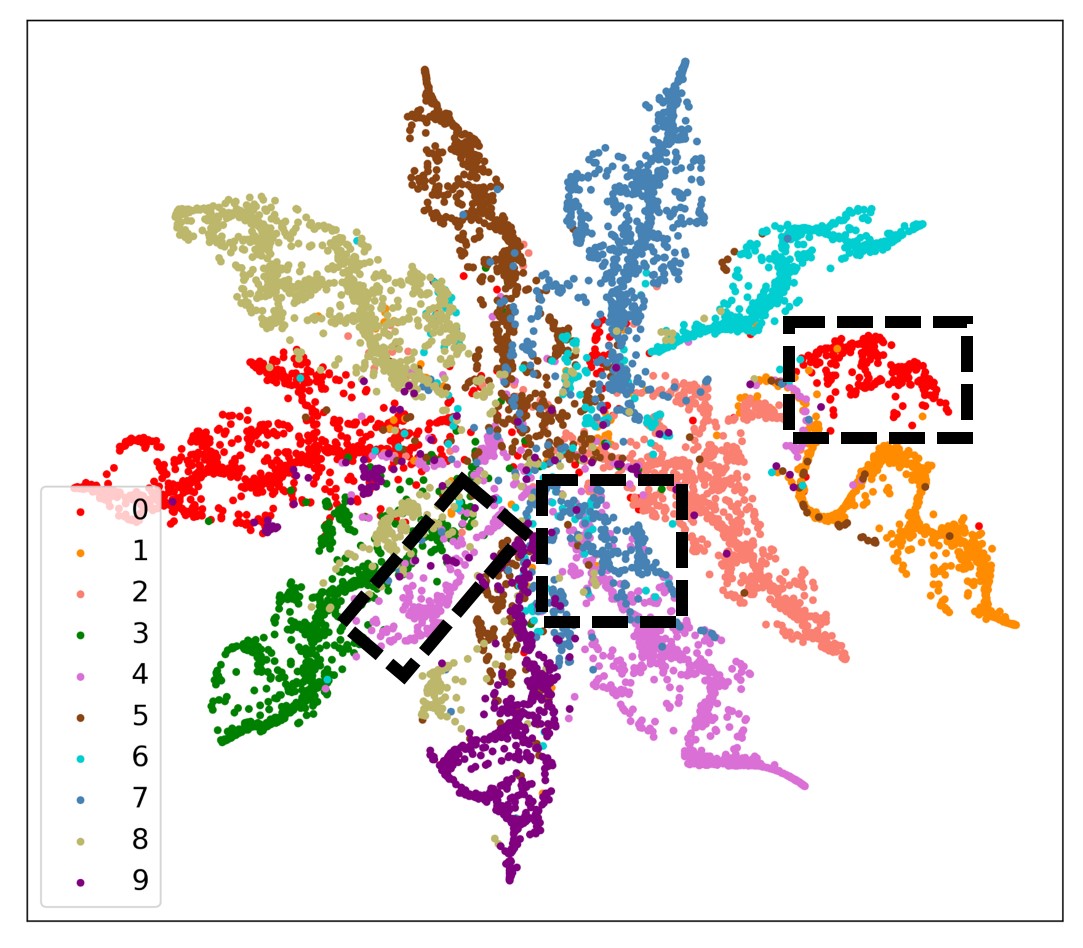}
        \caption{$\hat{y}$}
        \label{fig:tsne_upper}
    \end{subfigure}
    \begin{subfigure}[t]{0.242\columnwidth}
        \includegraphics[width=\textwidth]{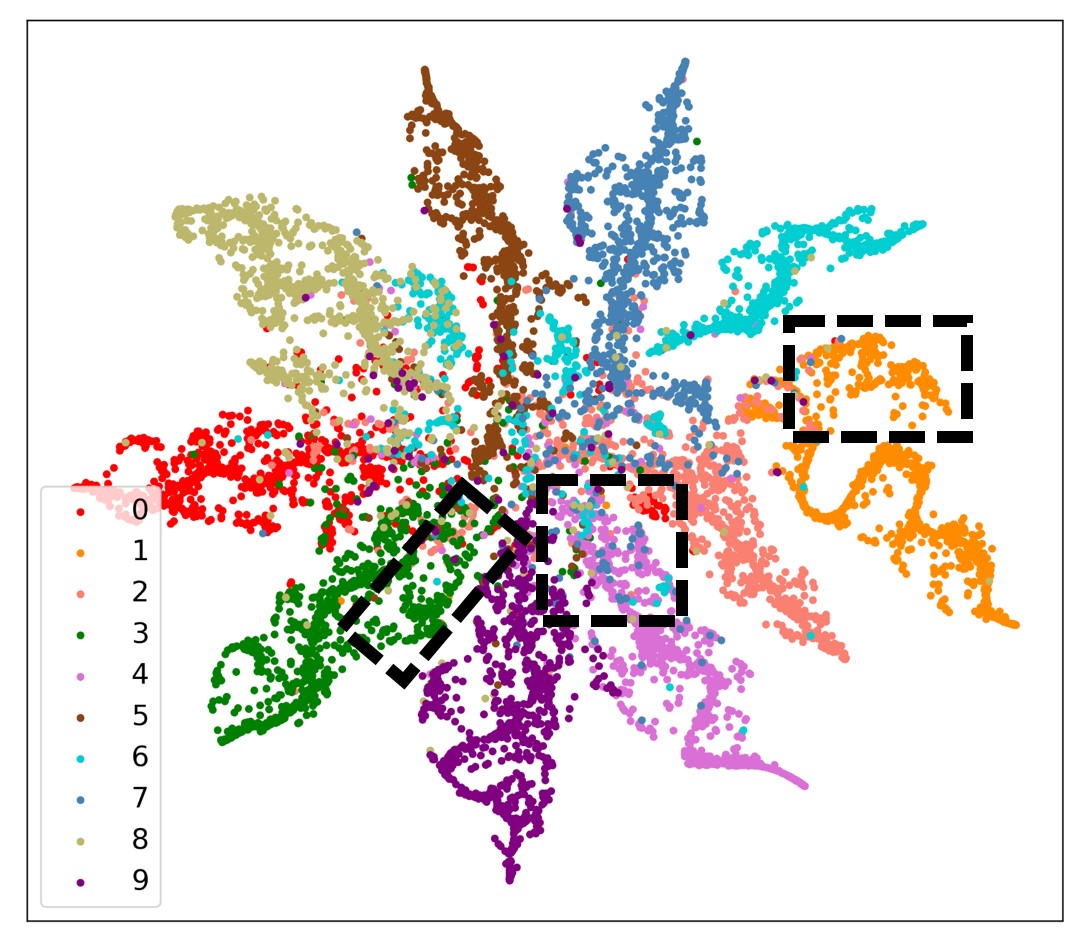}
        \caption{$y^{*}$}
        \label{fig:tsne_lower}
    \end{subfigure}
    \begin{subfigure}[t]{0.242\columnwidth}
        \includegraphics[width=\textwidth]{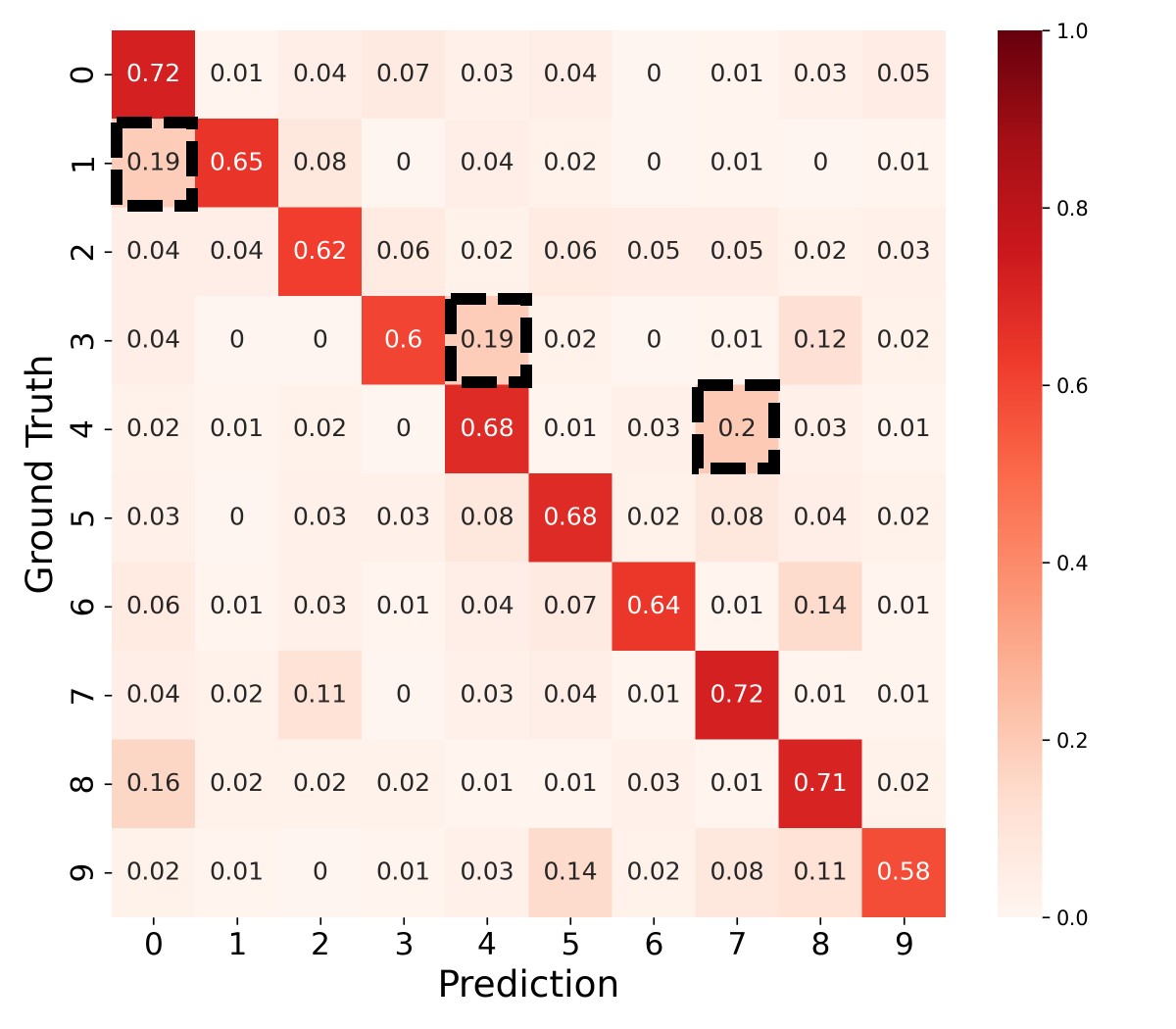}
        \caption{$\hat{y}\rightarrow y$}
        \label{fig:heatmap_upper}
    \end{subfigure}
    \begin{subfigure}[t]{0.242\columnwidth}
        \includegraphics[width=\textwidth]{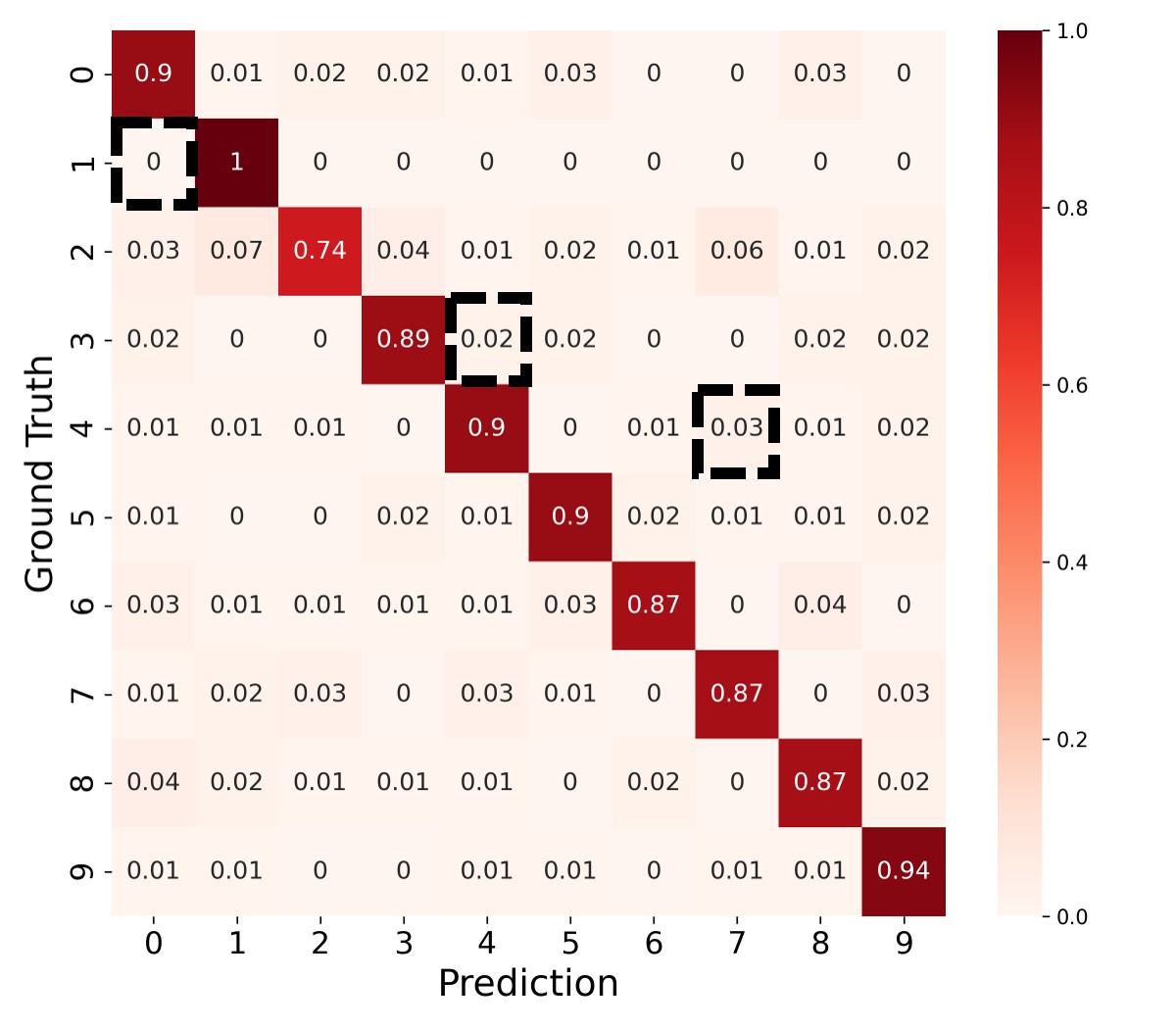}
        \caption{$y^{*}\rightarrow y$}
        \label{fig:heatmap_lower}
    \end{subfigure}
    \caption{t-SNE mapping view of our latent $q_{\phi}(y|\hat{y},x)$ (a, b) and corresponding confusion matrix (c, d). Colors of t-SNE views and columns of the confusion matrix represent $\hat{y}$ and $y^{*}$, respectively.}
\label{fig:tsne_comparison}
\end{figure}

Figure \ref{fig:heatmap_upper} and \ref{fig:heatmap_lower} illustrates the confusion matrix, each comparing the prediction of a classifier and the result after calibration with NPC to the true label. Comparing the two figures, we can see the intrusion cases of $\hat{y}$ in a cluster of $y$, which indicates the misclassified result from the original classifier\footnote{Another thing to focus is that the average of the diagonal terms of the confusion matrix is larger than $0.6$, indicating that $\hat{y}$ is at least better than the original noisy label, $\tilde{y}$.}, become fewer with the calibrated $y^{*}$ from NPC. In other words, with darker diagonal colours, we analyze that $y^{*}$ is more similar to true $y$ than $\hat{y}$.

By relating t-SNE figures and the confusion matrix together, we can implicitly find out the meaning of color change of dots from Figure \ref{fig:tsne_upper} to Figure \ref{fig:tsne_lower}. Black dashed boxes show visible differences between $\Hat{y}$ and $y^{*}$. For example, red dots ($\Hat{y}=0$) of the right side in Figure \ref{fig:tsne_upper} is wrongly classified as the confusion matrix in Figure \ref{fig:heatmap_upper}. On the contrary, the labels of those samples have changed to orange color ($y^{*}=1$), its true labels. These calibrations of classifier outputs are the reason behind the performance gain in Table \ref{tab:acc_synthetic} and \ref{tab:comparison_real}.

\subsection{NPC as a Post-processor}
\label{sec:post}

Our method, NPC, modifies predictions from the classifier after training. In this section, we first analyze the efficacy of NPC compared to other post-processing methods (\ref{sec:post_process}). Second, we discuss the similarities and differences between label correction framework and post-processing (\ref{sec:lc and pp}), We further discuss the empirical convergence of NPC from the iterative application (in the Appendix \ref{sec:iter npc}).

\subsubsection{Comparison with Post-processors}
\label{sec:post_process}

\begin{figure}[h]
\centering
\includegraphics[width=0.95\columnwidth]{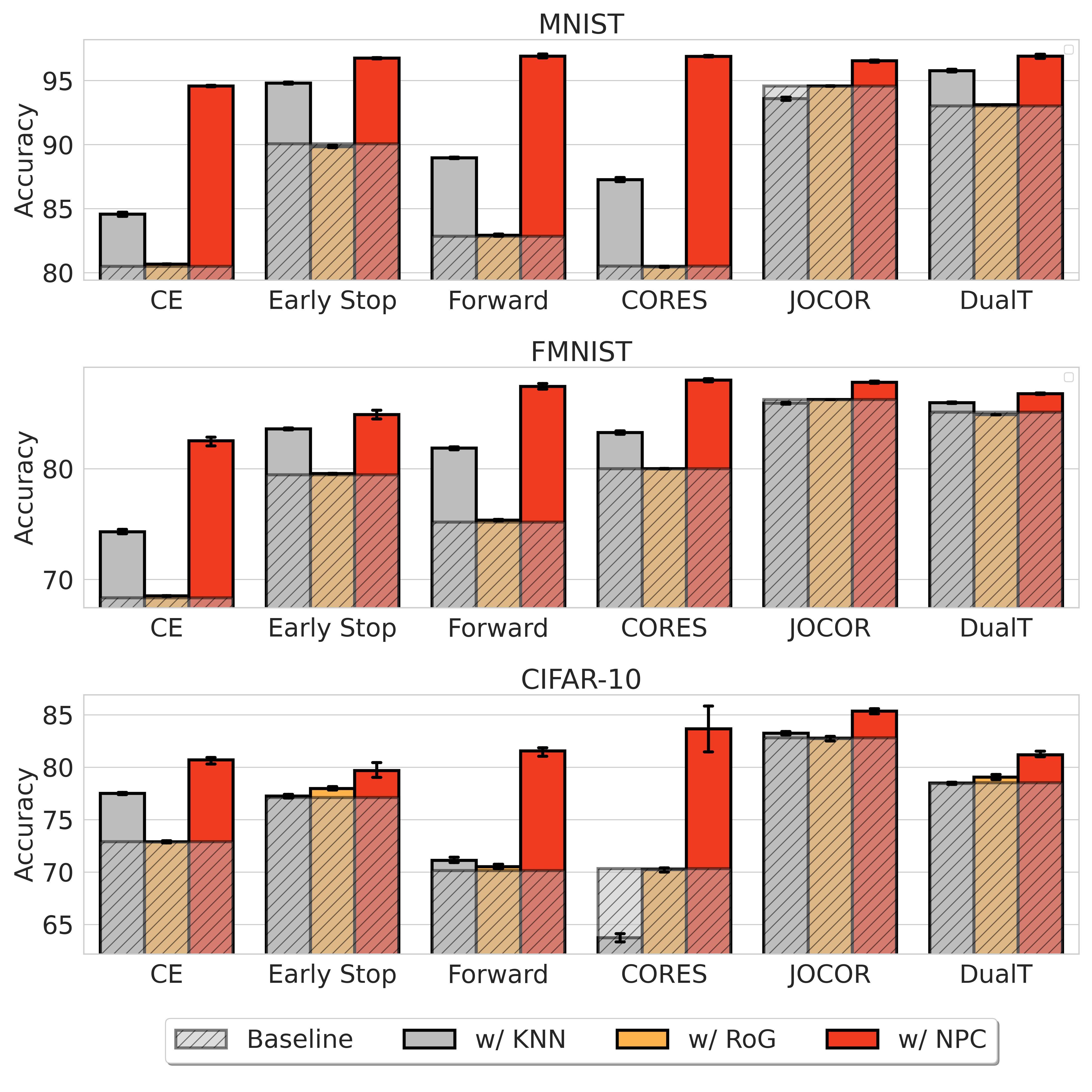}
\vskip-0.1in
\caption{Classification accuracy comparison applied on various baseline methods. Results are averaged over five trials with standard deviations. We denote the baseline as the algorithm utilized for training classifier.}
\label{fig:post_processing}
\vskip-0.1in
\end{figure}

In this paper, we do not recognize the methods which optimize or re-train classifier parameters during the inference, as post-processors. 
In other words, we treat the the methods as a post-processors only when they manage the relationship between input features and labels, without modifying the classifier parameters.

\noindent
\textbf{KNN Prior} KNN prior, which we utilize as a prior for inference of our posterior, could be recognized as a post-processor itself. Therefore, we compare the performance of KNN with NPC.

\noindent
\textbf{RoG} \cite{rog} is a well-known post processing method for noisy label classification, as well. RoG assumes that samples with wrong labels are mostly outliers, so RoG removes outliers from the measured mahalanobis-distance \cite{mahalanobis}, hypothesizing a classifier to be more robust to noisy labels. 

Figure \ref{fig:post_processing} indicates that NPC performs better than KNN and RoG on several baselines. Particularly, there is a significant gap between KNN and NPC, which emphasizes the importance of modeling posterior distribution via generative model, other than the prior setting on $y$. We also report the result on Food-101 and Clothing1M in Appendix \ref{appen:post_processing_real_dataset}.

\subsubsection{Label Correction and Post-processing}
\label{sec:lc and pp}
Label correction method, such as Joint \cite{joint}, LRT \cite{lrt} and MLC \cite{meta} have similarity with post-processing methods in that both utilizes the prediction of a given classifier to correct the noisy label. They are different, however, considering the classifier training: label correction methods iteratively correct the labels and train the classifier, which requires access to the classifier parameters. Post-processing methods are different from them because it does not entail classifier training. Having said that, by replacing noisy label $\tilde{y}$ with prediction $\hat{y}$ of a pre-trained classifier and training a new classifier, existing methods can be utilized as post-processor. Therefore, we focus on label correction methods and the utility of these methods as post-processor.

\begin{table}[h]
\centering
\resizebox{\columnwidth}{!}{%
\begin{tabular}{c|cccc|cccc}
\toprule Method & \multicolumn{4}{c|}{Label Correction} & \multicolumn{4}{c}{Post-processing} \\ \cmidrule(lr){0-1}\cmidrule(lr){2-5}\cmidrule(lr){6-9}
Noise & Joint& LRT& MLC&CauseNL&LRT$^{*}$&MLC$^{*}$&CauseNL$^{*}$&NPC\\\midrule
SN &80.0$\pm$\scriptsize{0.6} &82.9$\pm$\scriptsize{0.2}&71.1$\pm$\scriptsize{1.9}&77.2$\pm$\scriptsize{1.5}&82.7$\pm$\scriptsize{0.1}&82.2$\pm$\scriptsize{1.9}&83.5$\pm$\scriptsize{0.5}&\textbf{85.3}$\pm$\scriptsize{0.3}\\\midrule
IDN &78.6$\pm$\scriptsize{1.3} &82.5$\pm$\scriptsize{0.2}&72.2$\pm$\scriptsize{2.6}&78.4$\pm$\scriptsize{1.7}&82.9$\pm$\scriptsize{0.2}&82.1$\pm$\scriptsize{0.4}&83.3$\pm$\scriptsize{0.5}&\textbf{84.8}$\pm$\scriptsize{0.1}\\ \bottomrule
\end{tabular}%
}
\caption{Test accuracy on CIFAR-10 with different noise types. Noise ratio are $20\%$. Experiments are repeated over 5 times. For post-processing models, we utilize the prediction trained with Co-teaching method for fair comparison with CausalNL.}
\label{tab:comparison_with_post_processing}
\end{table}

Table \ref{tab:comparison_with_post_processing} shows the model performance of label correction methods and their applicability as post-processors. Baselines without and with an asterisk mean the original version and the post-processor of models, respectively. NPC shows best performance with statistical significance.

\subsection{NPC as a Generative Model}
\label{sec:causalNL}
\subsubsection{What CausalNL and NPC focus from $X$}

To understand what information each model focuses on from $X$ to predict $Y$, we utilize GradCAM method \cite{gradcam}. Figure \ref{fig:gradcam} shows that NPC focuses on the class-related features while CausalNL also pays attention to other parts, which is driven by the generation of $X$. For example, to classify an image of \textit{Churros} correctly, NPC mainly understands the part of features largely related to the \textit{Churros} itself, while CausalNL captures both the \textit{Churros} part and the \textit{Chocolate} part to generate the image itself also well.
\begin{figure}[h]
\centering
	\vskip-0.1in
	\includegraphics[width=0.99\columnwidth]{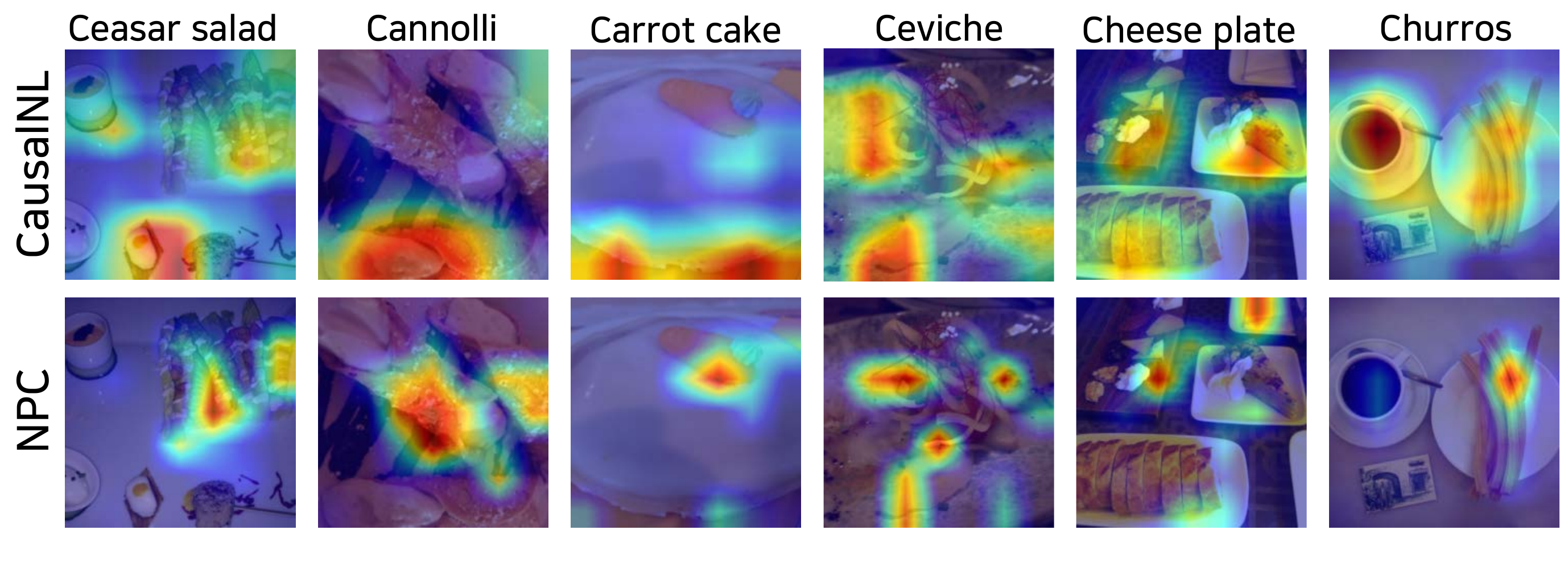}
	\vskip-0.1in
	\caption{GradCAM results on Food-101. CausalNL tends to capture the whole image, while NPC focuses on the class-related features. Texts on each Column denotes the label of each image.}
	\vskip-0.1in
	\label{fig:gradcam}
\end{figure}
\subsubsection{Failure Case Analysis}
\begin{figure}[h]
\centering
	\vskip-0.1in
	\includegraphics[width=0.99\columnwidth]{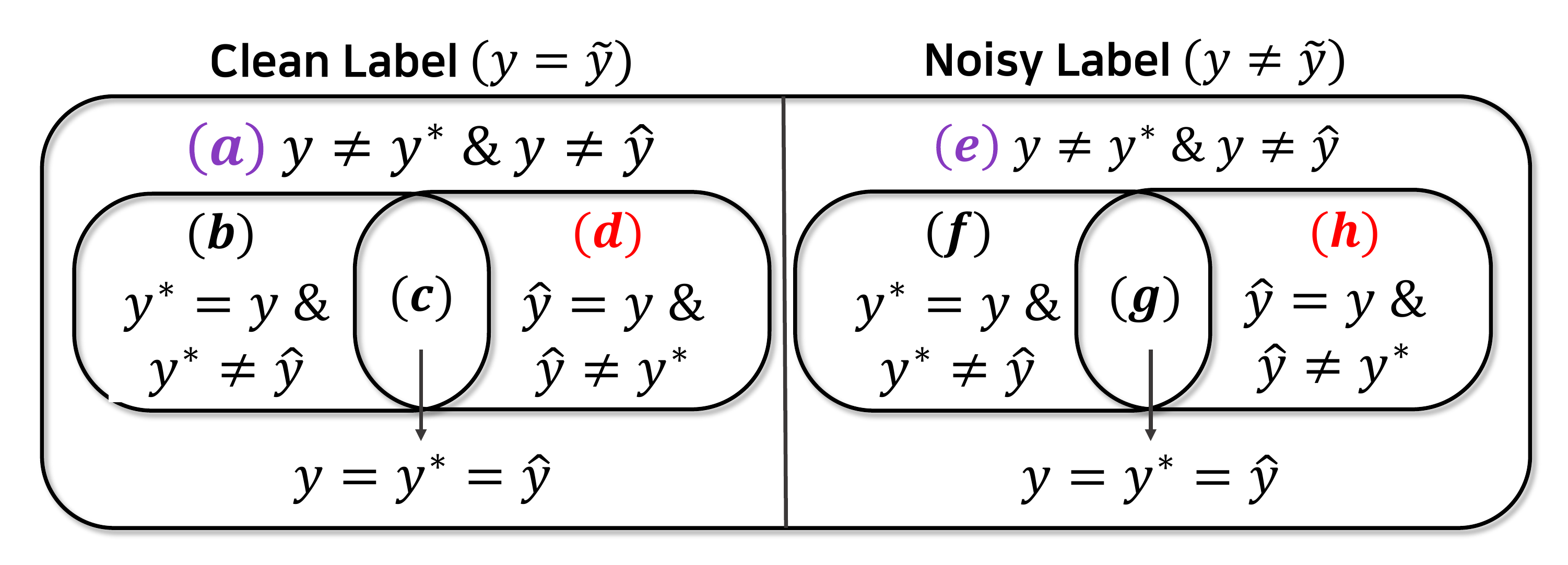}
	\vskip-0.1in
	\caption{Venn diagram by relation between true class ($y$), noisy label ($\tilde{y}$), classifier model prediction ($\hat{y}$) and model prediction after post-processing ($y^{*}$)}
	\label{fig:ven_diagram}
\end{figure}
Figure \ref{fig:ven_diagram} is a Venn-Diagram to illustrate the count on gains and losses from post-processing. An analysis on four cases are required to evaluate the performance of a post-processor; 1) miss from classifier and miss from post-processor $((a),(e))$, 2) miss from classifier and hit from post-processor $((b),(f))$, 3) hit from classifier and hit from post-processor $((c),(g))$, and 4) hit from classifier and miss from post-processor $((d),(h))$. Reducing case 4) as small as possible while increasing case 2) would indicate the best case for post-processors. Table \ref{tab:failure_case_analyses} shows the results on NPC and CausalNL$^{*}$, which is the post-processor version of CausalNL.

\begin{table}[h]
    \centering
    \resizebox{0.99\columnwidth}{!}{%
    \begin{tabular}{c | cccc|cccc}
        \toprule  
         & (a) & (b) & (c) & (d)& (e) & (f) & (g) & (h) \\ \midrule
        NPC &8 & 89 & 39799 &86 & 9035 & 949 &15 &19 \\
        CausalNL$^{*}$ &39 &58 &32459 &7426 &2446 &7538 &31 &3\\ \bottomrule
    \end{tabular}
    }
    \caption{The number of samples on each region for Venn diagram in Figure \ref{fig:ven_diagram}. Results from CIFAR-10 with Symmetric Noise (20${\%}$)}
    \vskip-0.1in
    \label{tab:failure_case_analyses}
\end{table}
Computing $\{(b)+(f)\}-\{(d)+(h)\}$, NPC and CausalNL$^{*}$ each shows 933, 167 counts, making NPC 5.59 times more effective than CausalNL$^{*}$. Comparing cases on NPC and CausalNL$^{*}$ explains the difference of modeling nature between two methods. With NPC managing the relation of noisy prediction and the true label given $X$, NPC becomes a cautious corrector. However, CausalNL explores wider range of $Y$ with $X$ generation included in training objective, making CausalNL$^{*}$ a risk taking corrector.
\subsection{NPC Identifies Potential Noises in Benchmarks}
\label{sec:realdata}
In this section, we provide the detection results of potentially noisy instances from the learning of NPC on benchmark datasets. We first capture all test samples whose label annotations and predictions from NPC are different. From them, we select 50 samples with highest prediction confidences and show 20 human-picked samples with marks of original annotations and predictions. From MNIST, we found out that images, which were captured by NPC, could be written differently from the original intention of annotators. Moreover, the samples from Fashion-MNIST are observed to have a possibility to be wrongly labeled, e.g. confusing features between Sneaker and Ankle boot.
\begin{figure}[h]
\centering
\includegraphics[width=0.95\columnwidth]{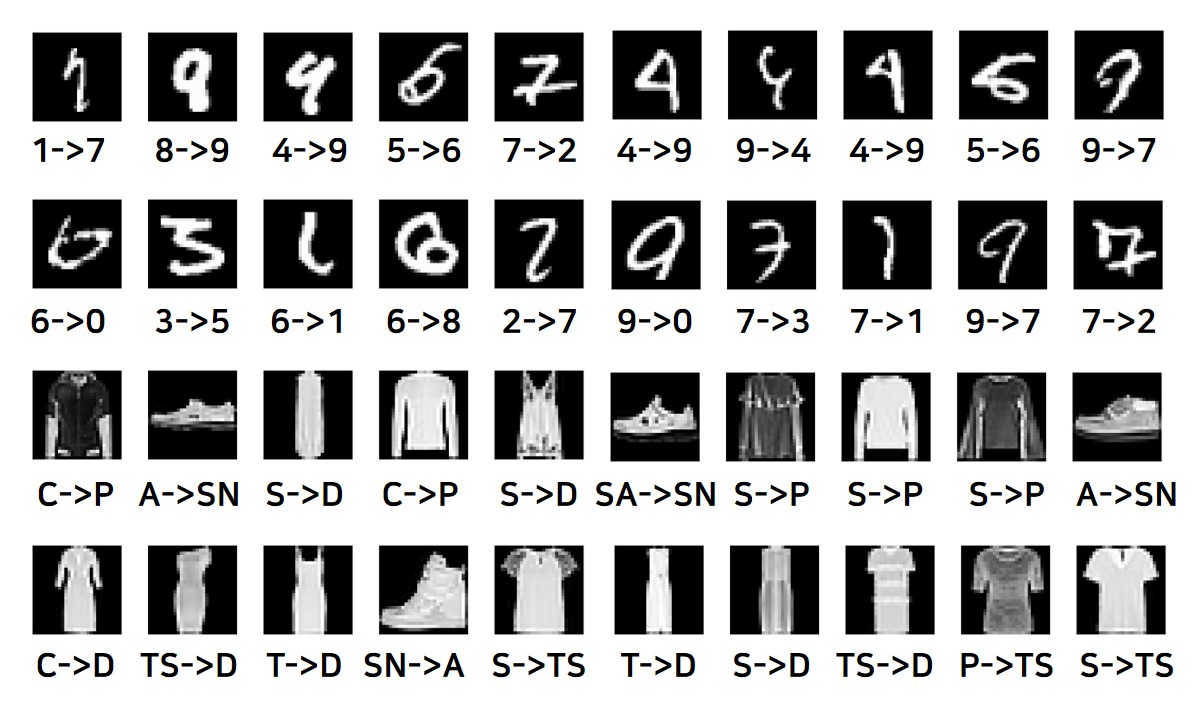}
\vskip-0.1in
\caption{The selected samples of MNIST (upper two rows) and FMNIST (under two rows) from the set of instances whose annotations and predictions from NPC are different. The marks below images denote (label $\rightarrow$ prediction). The abbreviation for Fashion-MNIST means as follows:
$\{$C: Coat, P: Pullover, A: Ankle boot, SN: Sneaker, D: Dress, SA: Sandal, S: Shirt,TS: T-shirt/top$\}$
    }
\vskip-0.15in
\label{fig:figuremnist}
\end{figure}
\section{Conclusion}
\label{sec:conclusion}
In this paper, motivated by the possible failure of the classifier trained with noisy labels, we provide a novel method, Noisy Prediction Calibration (NPC), to calibrate noisy prediction from a classifier to a true label. By explicitly modelling the relation between the output of a classifier and the true label, NPC opens up new possibilities for formulating noisy label problems. With the provided experiments on several types of noise settings, NPC proves its effectiveness to improve the robustness of a classifier's prediction in various situations. Also, we believe that our methodological framework, which calibrates the prediction of given classifier, can be actively utilized on other fields of machine learning, e.g., long tailed recognition, domain adaptation.

\section*{Acknowledgements}

This research was supported by AI Technology Development for Commonsense Extraction, Reasoning, and Inference from Heterogeneous Data(IITP) 
funded by the Ministry of Science and ICT(2022-0-00077)

\bibliography{example_paper}
\bibliographystyle{icml2022}

\newpage
\appendix
\onecolumn

\section{Previous Researches}
Here, we explain the details on the previous researches for learning with noisy labels which were not included in the main paper.

\subsection{Previous Researches for Learning With Noisy Labels}
\label{appen: prev noisy}
\subsubsection{Extraction of Reliable Clean Samples}
\cite{coteaching} is one of the studies which assumed that samples with small losses are to be clean. The problem of small loss criteria is that, however, the deep neural network will overfit to small loss samples in earlier learning stage, and it will cause learning bias. To prevent this problem, \cite{coteaching} use two identical structured networks with different initial point. Then it chooses small loss samples from each network and exchanges them with peer network for updating the parameters. With the increase of training epochs, however, two networks have converged to a consensus gradually. Therefore, \cite{coteachingplus} has been proposed. In this study, a sample can only be selected when its output from two different networks disagree and it has small loss. Although it solves the problem of two different networks' convergence to a same point, it selects very few examples to train classifiers, especially when the noise ratio is high. This phenomenon is reported in \cite{jocor}, and they relieve these limitations by updating two networks together, making the result of two networks become closer to true labels and peer network's.

\subsubsection{Label Modification}
\cite{joint} jointly optimizes both model parameter and label data, initialize network parameters and train with modified labels again. \cite{pencil} adopts label probability distributions to supervise network learning and to update these distributions through back-propagation end-to-end in each epoch. \cite{dividemix} removes labels of samples with large loss, considering the samples as unlabeled, and applying semi-supervised learning strategies. \cite{lrt} takes the likelihood ratio between the classifier’s confidence on noisy label and its confidence on its own label prediction as its threshold to configure clean labeled dataset, and corrects the label into the prediction for samples with low likelihood ratio iteratively. \cite{meta} uses meta learning and reweights samples depending on the cleanness of its label. \cite{proselflc} analyzes several types of label modification approaches and suggests label correction regularization hyperparameter depending both on learning time stage and confidence of a sample.

\subsubsection{Robust Loss and Regularization} 
\cite{mae} propose that the mean absolute loss (MAE) is more robust to the noisy label. However, it can cause underfitting, meaning a classifier converge to a bad sub-optimal. \cite{gce} combine the advantages of the mean absolute loss (MAE) and the cross entropy loss (CE) to obtain a better loss function and presents a theoretical analysis of the proposed loss functions in the context of noisy labels. \cite{sce} analyze CE, propose that CE fail to learn all classes uniformly well when learning with noisy labels. They suggest that adding reverse cross entropy term to original cross entropy term makes more robust loss to noisy label.

\cite{elr} is based on the phenomenon that the deep neural networks trained with noisy labels make progress during the early learning stage before memorization occurs. Therefore, they added a regularization term that seeks to maximize the inner product between the model output and the targets, with the target same as the weighted results of previous epochs. \cite{rel} distinguishes critical parameters from non-critical parameters by gradient value and use only critical parameters for fitting true labels, hoping to solve overparameterization problem. 

\subsection{Explanation on Recent Instance-dependent Transition Matrix Studies}
\label{appen: prev expla}
\subsubsection{PDN}
\cite{pdn} approximates $T(x)$ by a weighted combination of \textit{part-dependent} transition matrices, $\{P^{k}\}^{r}_{k=1}$. This estimation is elaborated as follows:
\begin{equation}
T_{i,j}(x) = \sum^{r}_{k=1}h_{k}(x)P_{i,j}^{k}
\label{eq:part_dependent}
\end{equation}
Here, $h_{k}(x)$ is a input-dependent scalar weight for $P^k$. However, this estimation requires another estimation of $P^k$ relying on \textit{anchor point}s, whose discrete selection becomes heuristic, i.e. selecting an instance with a confident output from noisy classifier $\tilde{f}$.

\subsubsection{CSIDN}
\cite{csidn} introduces a new method based on \textit{confidence score}, $p(y=i|\Tilde{y}=i,x)$, and this method assumes that the confidence score is given. With observed $p(y=i|\Tilde{y}=i,x)$, $T$ is transformed as Eq. \ref{eq:CSIDN}.
\begin{equation}
\begin{aligned}
T_{i,j}(x)& = \left\{\begin{matrix}
P(y=i|\Tilde{y}=i,x)\frac{P(\Tilde{y}=i|x)}{P(y=i|x)} & \; if \; j=i\\ 
P(\Tilde{y}=j|\Tilde{y}\neq i, y=i)(1-T_{i,i}(x)) & \; if \; j \neq i
\end{matrix}\right.
\end{aligned}
\label{eq:CSIDN}
\end{equation}
Whereas this method provides an input-dependent $T$, an assumption of observing $p(y=i|\Tilde{y}=i,x)$ may be unrealistic and potentially annotation intensive. 
\section{Reparameterization Trick for Dirichlet Distribution}
\label{appen: reparameterization}
Following \cite{dirichletvae}, we consider the dirichlet distribution as a composition of gamma random variables as follows:
\begin{equation}
    Dirichlet(x;\alpha) = \frac{\Gamma(\sum\alpha_{k})}{\prod\Gamma(\alpha_{k})}\prod x_{k}^{\alpha_{k}-1}, \; Gamma(x;\alpha, \beta) = \frac{\beta^{\alpha}}{\Gamma(\alpha)}x^{\alpha-1}e^{-\beta x}
\label{eq:dirigamma}
\end{equation}
where $\alpha_{k}, \alpha, \beta > 0$. In other words, if there are $K$ independent random variables $X_k \sim \text{Gamma} (\alpha_k , \beta)$ with $\alpha_{k}, \beta > 0$ for $k = 1, \cdots, K$, we have $Y \sim \text{Dirichlet} (\alpha)$ where $Y_k = X_k / \sum X_i$. Note that $\beta$ should be same for all $X_k$. In this way, we can calculate the kl divergence of two dirichlet distributions as kl divergence of two multi-gamma distributions, where $X = (X_1, \cdots, X_K) \sim \text{MultiGamma}(\alpha, \beta \cdot \mathbf{1}_K)$ represents vector of $K$ independent gamma random variables $X_k \sim \text{Gamma} (\alpha_k , \beta)$ with $\alpha_k , \beta > 0$ for $k = 1, \cdots, K$.

Note that the derivative of a Gamma-like function $\frac{\Gamma(\alpha)}{\beta^\alpha}$ can be derived as follows:

\begin{equation}
\frac{d}{d\alpha} \frac{\Gamma (\alpha)}{\beta^{\alpha}} = \beta^{-\alpha} (\Gamma ^\prime (\alpha) - \Gamma (\alpha) \log \beta) = \int_0 ^{\infty} x^{\alpha -1} e^{- \beta x} \log x ~ dx ~ .
\label{eq:gamma derivative}
\end{equation}

Using Eq. \ref{eq:gamma derivative}, the KL divergence can be written as follows:

\begin{equation}
\begin{aligned}
~ & \text{KL}(Q || P) = 
 \int_{\mathcal{D}} q(\mathbf{x}) \log \frac{q(\mathbf{x})}{p(\mathbf{x})} ~ d\mathbf{x} = \int_0 ^{\infty} \cdots \int_0 ^{\infty} \prod \text{Gamma} (\hat{\alpha}_k ,\beta) \log \frac {\beta ^{\sum \hat{\alpha}_k \prod \Gamma^{-1} (\hat{\alpha}_k) e^{-\beta \sum x_k} \prod x_k ^{\hat{\alpha}_k -1}}}{\beta ^{\sum {\alpha}_k \prod \Gamma^{-1} ({\alpha}_k) e^{-\beta \sum x_k} \prod x_k ^{{\alpha}_k -1}}} ~ d\mathbf{x} \\
= & \int_0 ^{\infty} \cdots \int_0 ^{\infty} \prod \text{Gamma} (\hat{\alpha}_k ,\beta) \times \Big[ \sum ( \hat{\alpha}_k - \alpha_k ) \log \beta + \sum \log \Gamma (\alpha_k) - \sum \log \Gamma (\hat{\alpha}_k) + \sum (\hat{\alpha}_k - \alpha_k ) \log x_k \Big] ~ d\mathbf{x} \\
= & \sum ( \hat{\alpha}_k - \alpha_k ) \log \beta + \sum \log \Gamma (\alpha_k) - \sum \log \Gamma (\hat{\alpha}_k) + \int_0 ^{\infty} \cdots \int_0 ^{\infty} \frac{\beta ^{\hat{\alpha}_k}}{\prod \Gamma (\hat{\alpha}_k)} e^{-\beta \sum x_k} \prod x_k ^{\hat{\alpha}_k -1} \big( \sum (\hat{\alpha}_k - \alpha_k ) \log x_k \big) ~ d\mathbf{x} \\
= &  \sum ( \hat{\alpha}_k - \alpha_k ) \log \beta + \sum \log \Gamma (\alpha_k) - \sum \log \Gamma (\hat{\alpha}_k) + \sum (\hat{\alpha}_k - \alpha_k ) \beta^{\hat{\alpha}_k} \Gamma^{-1} (\hat{\alpha}_k) \beta^{-\hat{\alpha}_k} \big( \Gamma^\prime (\hat{\alpha}_k) - \Gamma (\hat{\alpha}_k) \log \beta \big) \\
= & \sum ( \hat{\alpha}_k - \alpha_k ) \log \beta + \sum \log \Gamma (\alpha_k) - \sum \log \Gamma (\hat{\alpha}_k) + \sum (\hat{\alpha}_k - \alpha_k)(\psi (\hat{\alpha}_k) - \log \beta) \\
= & ~ \sum \log \Gamma (\alpha_k) - \sum \log \Gamma (\hat{\alpha}_k) + \sum (\hat{\alpha}_k - \alpha_k) \psi (\hat{\alpha}_k) 
\end{aligned}
\end{equation}

Now we can calculate the kl divergence of two multi-gamma distributions as follows:

\begin{equation}
\text{KL} (Q||P) = \sum \log \Gamma (\alpha_k) - \sum \log \Gamma (\hat{\alpha}_k) + \sum (\hat{\alpha}_k - \alpha_k) \psi (\hat{\alpha}_k) ~,
\end{equation}

where $P = \text{MultiGamma} (\alpha, \beta \cdot \mathbf{1}_K)$ and $Q = \text{MultiGamma} (\hat{\alpha}, \beta \cdot \mathbf{1}_K)$.

\newpage
\section{Alignment of $T$ and $H$}

\subsection{Proof of Proposition \ref{pro:transition relation}}
\label{appen: proposition proof}
\begin{proposition}
Assume that $\widehat{y}$ and $y$ are conditional independent given $\widetilde{y}$, and $p(\widehat{y}=k|x) \neq 0$ for all $k=1,..,c$. Then, $H_{kj}(x) = \frac{p(y=j|x)}{p(\widehat{y}=k|x)}\sum_{i}p(\widehat{y}=k|\widetilde{y}=i,x)T_{ij}(x)$ for all $j,k=1,...,c$.
\end{proposition}

\begin{proof}
\begin{align}
p(y=j|\Hat{y}=k,x)&=\sum_{i}p(y=j,\Tilde{y}=i|\Hat{y}=k,x)\\
&=\sum_{i}\frac{p(y=j,\Tilde{y}=i,\Hat{y}=k|x)}{p(\Hat{y}=k|x)}\\
&=\sum_{i}\frac{p(\Hat{y}=k|\Tilde{y}=i,y=j,x)p(\Tilde{y}=i|y=j,x)p(y=j|x)}{p(\Hat{y}=k|x)}\\
&=\frac{p(y=j|x)}{p(\Tilde{y}=i|x)} \sum_{i} p(\Hat{y}=k|\Tilde{y}=i,x)p(\Tilde{y}=i|y=j,x) \quad (\because \widehat{y} \bigCI y | \widetilde{y}) \\
\end{align}
Therefore, $H_{kj}(x) = \frac{p(y=j|x)}{p(\widehat{y}=k|x)}\sum_{i}p(\widehat{y}=k|\widetilde{y}=i,x)T_{ij}(x)$ for all $j,k=1,...,c.$ \\
\end{proof}

The assumption of $\hat{y} \indep y|\tilde{y}$ is natural because $\hat{y}$ is conditionally independent to $y$ when $\hat{\psi}$ is trained only with $\tilde{y}$. Proposition \ref{pro:transition relation} shows that $H$ can be formulated by $T$ with weighting variables, which are observable or can be easily computed from our framework. It implies that we can also infer $T$ from the inference procedure of $H$, which reduces the framework of NPC to the transition-based approaches.

\subsection{Experiment on the Relation between $T$ and $H$}
\label{appen: t and h}
\subsubsection{Experiment Details}
For $p(\hat{y}|x)$, we use same network as in Section \ref{sec:exper}. For finding $p(y|x)$, we use our model output, which is $\sum_{k=1}^c p(y|\hat{y}=k,x)p(\hat{y}=k|x)$. What is left is then $p(\hat{y}|\Tilde{y},x)$ term. Since all variables for the probability calculation is given as datasets, we can make a function that approximates this probability with another deep neural network. This network has similar structure with the networks used in Section \ref{sec:exper}. The only difference is $\Tilde{y}$ is to be concatenated with feature of $x$. For other implementation details, such as batch size, learning rate, we utilize same condition as in Section \ref{sec:exper}.

In the concept of instance-dependent noisy label generation model, each transition matrix of all different samples is assumed not to be identical. As a method to merge instance-wise different transition matrices as a single representative value, we approximate the class-wise transition matrix with monte-carlo estimation, which is calculated as Eq. \ref{eq:mc}.
\begin{equation}
\begin{aligned}
T_{ij} &= P(\tilde{y}=j|y=i)\\
&= \sum_{x} P(\tilde{y}=j,x|y=i)\\
&= \sum_{x} P(\tilde{y}=j|y=i,x)P(x|y=i)\\
&\approx \frac{1}{N_{i}} \sum_{x} P(\tilde{y}=j|y=i,x),\ where \ N_{i} = total\ number\ of\ samples\ of\ class\ i\\
\end{aligned}
\label{eq:mc}
\end{equation}

\subsubsection{Additional Results on Transition Matrix Estimation}

We report MSE between the transition matrix approximated from each algorithm and the true transition matrix in Table \ref{tab: transition matrix mse} for CIFAR-10 dataset with various noisy label conditions. The result shows that NPC approximates the transition matrix as good as the traditional algorithms.

\begin{table}[h]
\centering
\begin{tabular}{c|c|ccccc}
\toprule
Noise & Ratio   & Forward & DualT  & TVR    & CausalNL & NPC    \\ \midrule
\multirow{2}{*}{SN}  & 20 & 0.0018  & 0.0012 & 0.0018 & 0.0035   & 0.0015 \\ \cmidrule{2-7}
& 80 & 0.0004  & 0.0576 & 0.0005 & 0.0195   & 0.0005 \\ \midrule
\multirow{2}{*}{IDN} & 20 & 0.0038  & 0.0021 & 0.0024 & 0.0034   & 0.0035 \\ \cmidrule{2-7}
 & 40 & 0.0041  & 0.0043 & 0.0032 & 0.0049   & 0.0018 \\
 \bottomrule
\end{tabular}
\caption{MSE between the approximated transition matrix and the true one.}
\label{tab: transition matrix mse}
\end{table}

\section{Experiment}
Here, we manage dataset explanation, noise label generation process, experiment settings and additional result.

\subsection{Implementation Details and Baseline Description}

\subsubsection{Synthetic Noisy Label Generation Process}
\label{appen : noisy label generation}

\textit{MNIST} \cite{mnist} and \textit{Fashion-MNIST} \cite{fmnist} are both $28\times28$ grayscale image datasets with 10 classes, which include 60,000 training samples and 10,000 test samples. \textit{CIFAR-10} \cite{cifar10} is $32\times32\times3$ color image dataset with 10 classes, which includes 50,000 training samples and 10,000 test samples. Since these datasets are assumed to have no noisy labels, we manage four types of noisy label for noisy label injection. Here, we explain details of those noisy label generation processes. Since we explain enough for symmetric noise in Section \ref{sec:exper}, we pass explanation on it.

\textbf{Asymmetric Noise (ASN)} \cite{coteaching, joint, rel}
For this type of noise, we flipped label class as below, following the previous researches.

\begin{itemize}
    \item MNIST : $2 \Rightarrow 7, 3 \Rightarrow 8, 5 \Leftrightarrow 6$
    \item FMNIST : $T-shirt \Rightarrow Shirt, Pullover \Rightarrow Coat, Sandals \Rightarrow  Sneaker$
    \item CIFAR-10 : $Truck \Rightarrow Automobile, Bird \Rightarrow Airplane, Deer \Rightarrow Horse, Cat \Leftrightarrow Dog$ 
\end{itemize}

\textbf{Instance Dependent Noise (IDN)} We followed noise generation process as utilized at \cite{pdn,cores,causalNL}.
\begin{algorithm}
\caption{Instance Dependent Noise Generation Process}
\begin{algorithmic}[1]
	\REQUIRE Clean samples ${(x_i, y_i)_{i=1}^n}$;, Noise rate $\tau$ 
	\STATE Sample instance flip rates $q\in \mathbb{R}^n$ from the truncated normal distribution $\textit{N}(\tau, 0.1^2, [0,1])$;
	\STATE Independently sample $w_1, w_2,..., w_c$ from the standard normal distribution $\textit{N}(0, 1^2)$;
	\FOR{$i=1,2,...,n$}
		\STATE $p=x_i\times w_{y_i}$;
		\STATE $p_{y_i}=-\inf $;
		\STATE $p=q_i\times softmax(p)$;
		\STATE $p_{y_i}=1-q_i$;
		\STATE Randomly choose a label from the label space according to the possibilities $p$ as noisy label $\bar{y}_i$;
	\ENDFOR
\end{algorithmic}
$\textbf{Output:}$ Noisy samples ${(x_i, \bar{y}_i)_{i=1}^n}$
\label{alg: idn noise generation}
\end{algorithm}

\textbf{Similarity Reflected Instance Dependent Noise (SRIDN)} We followed noise generation process as utilized at \cite{aaai-idn, csidn}.

\begin{algorithm}
\caption{Similarity Reflected Instance Dependent Noise Generation Process}
\begin{algorithmic}[1]
	\REQUIRE Clean samples ${(x_i, y_i)_{i=1}^n}$;, Noise rate $\tau$ 
	\STATE Train a classifier;
	\STATE Get output from a classifier $f_i \in \mathbb{R}^c$ for all $i=1,...,n$ and sort by $f_i^{y_i}$;
	\STATE Set $N_{noisy}=0$;
	\WHILE{$N_{noisy}<n\times \tau$}
		\STATE From the least confident sample, choose noisy label $\bar{y}_i=max_{j\neq y_i, j\in{1,...,c}}f_i^j$;
		\STATE $N_{noisy}=N_{noisy}+1$;
	\ENDWHILE
\end{algorithmic}
$\textbf{Output:}$ Noisy samples ${(x_i, \bar{y}_i)_{i=1}^n}$
\label{alg: sridn noise generation}
\end{algorithm}

We only perform normalization for CIFAR-10.

\subsubsection{Real Datasets with Noisy Labels}
\label{appen:real_data}
 \textit{Food101} \cite{food} is color image datasets with 101 food categories, each category having 1,000 samples each. For each class, 250 images are annotated by humans, consisting of test dataset and 750 images are provided with real-world label noise. \textit{Clothing-1M} \cite{clothing1m} is another real-world noisy label dataset collected from several online shopping websites. The data contains 1 million training images with 14 classes. We use the total 75,000 human annotated images only for testing. All datasets have been widely used in the previous researches \cite{jocor, rel, cores,causalNL}. For Food101 and Clothing1M, we resize the image to $256 \times 256$, crop the middle $224 \times 224$ as input, and perform normalization.

\subsubsection{Implementation Details}
\label{appen:experiment condition}
For fair comparison, we implement same network structures for all baseline methods. In details, we use network with two convolution layers for MNIST and Fashion-MNIST, 9-layered convolutional neural network for CIFAR-10 as in \cite{coteaching, coteachingplus}. For real-world dataset, Food101 and Clothing-1M, we use a ResNet-50 network pre-trained on ImageNet \cite{imagenet}.

\begin{table}[h]
    \centering
    \begin{tabular}{|c|c|}
    \toprule
    CNN on MNIST \& Fashion-MNIST & CNN on CIFAR-10 \\ \hline
    $28\times28$  & $32\times32\times3$         \\ \hline
     & $3\times3$ conv, 128LRELU  \\ 
     & $3\times3$ conv, 128LRELU  \\
     & $3\times3$ conv, 128LRELU  \\ 
     & $2\times2$ max-pool, stride 2 \\
     & dropout, $p=0.25$ \\ \cline{2-2} 
    $3\times3$ conv, 8RELU & $3\times3$ conv, 256LRELU  \\
    $3\times3$ conv, 16Tanh & $3\times3$ conv, 256LRELU  \\
    dense $16\times28\times28 \rightarrow 28\times28$ & $3\times3$ conv, 256LRELU  \\  
    dense $28\times28 \rightarrow 256$ & $2\times2$ max-pool, stride 2 \\
     & dropout, $p=0.25$ \\ \cline{2-2} 
     & $3\times3$ conv, 512LRELU  \\
     & $3\times3$ conv, 256LRELU  \\
     & $3\times3$ conv, 128LRELU  \\
     & avg-pool \\ \hline
    dense $256\rightarrow10$ & dense $128\rightarrow10$ \\ \bottomrule
    \end{tabular}
    \caption{Convolutional Neural Network Structure for MNIST, Fashion-MNIST and CIFAR-10 datasets.}
    \label{tab:network}
\end{table}

We train all synthetic datasets with batch size 128 and all real-world datasets with batch size 32. For all datasets, we use Adam optimizer with learning rate of $10^{-3}$ and no learning rate decay is applied. All methods are implemented by PyTorch.

\subsubsection{Baseline Description}
We compare the proposed method with the following approaches.
\label{baselines}
\begin{itemize}
    \item \textbf{CE} It trains the deep neural networks with the cross entropy loss on noisy datasets.
    \item \textbf{Joint} \cite{joint} It jointly optimizes the network parameters and the sample labels. We set $\alpha$ and $\beta$ as 1.0 and 0.5 respectively.
    \item \textbf{Coteaching} \cite{coteaching} It trains two different networks and the samples with small loss are only fed to the other network for learning. 
    \item \textbf{JoCoR} \cite{jocor} It also trains two networks and selects the samples, for which the sum of the losses from two networks is small, as clean samples. Following authors of the original paper, we set $\lambda$ as 0.5, and set other hyperparameter settings as cited on the paper.
    \item \textbf{CORES2} \cite{cores} Regarding the prediction output of a network, which is trained on noisy dataset, as a confidence of each instance, it selects data instances with high confidence as clean instances. We follow hyperparameter settings as cited on the paper.
    \item \textbf{SCE} \cite{sce} It trains network based on the specialized loss function, which is sum of cross entropy and reverse cross entropy. We follow hyperparameter settings as cited on the paper.
    \item \textbf{ES (EarlyStop)} Splitting a part of noisy training dataset as validation dataset, it only learns network until the validation performance continues to decrease.
    \item \textbf{LS (Label Smoothing)} \cite{ls} It trains network based on the noisy dataset where each label is expressed as a smoothed label by a given label smoothing factor rather than one-hot encoding. We set smoothing factor as 0.1.
    \item \textbf{REL} \cite{rel} To avoid the over-parameterization of the network, it only trains critical parameters, whose gradients are high, of network.
    \item \textbf{Forward} \cite{forward} It trains network based on the loss function, which is modified by the estimated transition probability matrix.
    \item \textbf{DualT} \cite{dualT} In order to reduce the estimation error of the transition probability matrix, the corresponding matrix is expressed as a product of two matrices, and each matrix is estimated.
    \item \textbf{TVR} \cite{tvr} To solve the problem of infinitely many possible solutions of the transition probability matrix, it minimizes the total variance distance. 
    \item \textbf{CausalNL} \cite{causalNL} We already explained on this research previously in Section \ref{sec:prev trans}.
\end{itemize}

\newpage
\subsection{Performance on MNIST Dataset}
\label{appen:MNIST result}
Due to space issue, we show test accuracy on MNIST dataset with several noise conditions here. In most noise settings, our proposed method increases model performance, calibrating noisy prediction to true label.

\begin{table}[h]
\centering
\resizebox{0.52\textwidth}{!}{%
\begin{tabular}{c ccccccccc}
\toprule

\multicolumn{1}{c}{\textbf{Model}}
 & \textbf{Clean} & \multicolumn{2}{c}{\textbf{SN}} & \multicolumn{2}{c}{\textbf{ASN}} & \multicolumn{2}{c}{\textbf{IDN}} & \multicolumn{2}{c}{\textbf{SRIDN}} \\
 \cmidrule(lr){2-2}\cmidrule(lr){3-4}\cmidrule(lr){5-6}\cmidrule(lr){7-8}\cmidrule(lr){9-10}
  & - & 20 & 80 & 20 & 40 & 20 & 40 & 20 & 40 \\ \midrule 
CE         & 97.8 & 86.3 & 29.7 & 89.2 & 81.3 & 80.5 & 66.3 & 82.9 & 65.4 \\
\rowcolor{Gray}w/NPC      & \textbf{98.2} & \textbf{95.3} & \textbf{36.6} & \textbf{96.3} & \textbf{92.3} & \textbf{94.6} & \textbf{89.1} & \textbf{85.2} & \textbf{70.3} \\\hline
Joint      & 93.0 & 93.9 & \textbf{22.0} & 93.8 & 93.2 & 93.1 & 93.6 & 87.1 & 81.1 \\
\rowcolor{Gray}w/NPC      & \textbf{94.5} & \textbf{95.5} & 21.9 & \textbf{95.2} & \textbf{94.9} & \textbf{94.8} & \textbf{95.5} & \textbf{89.1} & \textbf{84.4} \\\hline
Coteaching & 98.0 & 91.8 & 69.6 & 97.9 & 97.5 & 90.7 & 87.5 & 89.0 & 81.3 \\
\rowcolor{Gray}w/NPC      & \textbf{98.3} & \textbf{95.0} & \textbf{74.8} & \textbf{98.2} & \textbf{97.9} & \textbf{94.9} & \textbf{93.3} & \textbf{91.0} & \textbf{84.4} \\\hline
JoCoR      & 97.8 & 95.0 & \textbf{36.6} & \textbf{97.5} & \textbf{95.3} & 94.6 & 93.3 & 90.2 & 82.2 \\
\rowcolor{Gray}w/NPC      & \textbf{98.3} & \textbf{96.4} & 36.3 & 97.3 & 94.6 & \textbf{96.9} & \textbf{96.1} & \textbf{91.9} & \textbf{85.2} \\\hline
CORES2     & 97.0 & 75.3 & 16.6 & 88.5 & \textbf{9.8}  & 80.5 & 48.8 & 86.5 & 68.5 \\
\rowcolor{Gray}w/NPC      & \textbf{97.9} & \textbf{86.9} & \textbf{23.9} & \textbf{94.5} & \textbf{9.8}  & \textbf{90.7} & \textbf{67.4} & \textbf{89.4} & \textbf{78.5} \\\hline
SCE        & 97.7 & 85.7 & 30.5 & 89.4 & 81.3 & 80.6 & 66.6 & 83.3 & 65.8 \\
\rowcolor{Gray}w/NPC      & \textbf{98.2} & \textbf{95.1} & \textbf{38.5} & \textbf{96.2} & \textbf{92.5} & \textbf{94.5} & \textbf{88.8} & \textbf{85.3} & \textbf{70.7} \\\hline
Early Stop & 96.5 & 95.9 & 57.8 & 94.9 & 84.4 & 90.1 & 73.3 & 85.8 & 71.4 \\
\rowcolor{Gray}w/NPC      & \textbf{97.9} & \textbf{97.1} & \textbf{75.0} & \textbf{97.4} & \textbf{89.9} & \textbf{96.6} & \textbf{90.6} & \textbf{88.6} & \textbf{76.5} \\\hline
LS         & 97.8 & 85.8 & 32.6 & 88.9 & 82.4 & 80.6 & 66.2 & 84.3 & 65.3 \\
\rowcolor{Gray}w/NPC      & \textbf{98.2} & \textbf{95.1} & \textbf{40.6} & \textbf{96.0} & \textbf{92.8} & \textbf{94.4} & \textbf{88.7} & \textbf{86.2} & \textbf{70.8} \\\hline
REL        & \textbf{98.0} & 96.6 & 79.4 & 93.8 & 94.2 & 95.5 & 90.7 & 90.4 & 86.2 \\
\rowcolor{Gray}w/NPC      & 97.9 & \textbf{96.9} & \textbf{84.9} & \textbf{96.8} & \textbf{96.4} & \textbf{96.8} & \textbf{95.5} & \textbf{91.7} & \textbf{88.7} \\\hline
Forward    & 98.0 & 88.1 & 27.8 & 91.8 & 82.1 & 82.9 & 67.9 & 84.5 & 67.7 \\
\rowcolor{Gray}w/NPC      & \textbf{98.4} & \textbf{95.6} & \textbf{39.4} & \textbf{97.9} & \textbf{96.9} & \textbf{96.1} & \textbf{91.1} & \textbf{86.9} & \textbf{74.4} \\\hline
DualT      & 96.7 & 93.6 & \textbf{9.8}  & 95.9 & 89.7 & 93.0 & 94.3 & 86.6 & 72.8 \\
\rowcolor{Gray}w/NPC      & \textbf{97.8} & \textbf{96.3} & \textbf{9.8}  & \textbf{97.7} & \textbf{93.1} & \textbf{95.9} & \textbf{96.5} & \textbf{89.3} & \textbf{77.3} \\\hline
TVR        & 97.7 & 84.7 & 31.8 & 86.7 & 80.0 & 75.2 & 64.4 & 83.7 & 66.3 \\
\rowcolor{Gray}w/NPC      & \textbf{98.1} & \textbf{94.4} & \textbf{38.9} & \textbf{95.4} & \textbf{91.0} & \textbf{91.6} & \textbf{84.5} & \textbf{85.7} & \textbf{71.3} \\\hline
CausalNL   & 98.1 & 94.1 & 57.1 & \textbf{98.3} & 97.4 & 92.7 & 85.2 & 84.5 & 68.9 \\
\rowcolor{Gray}w/NPC      & \textbf{98.4} & \textbf{96.9} & \textbf{63.5} & \textbf{98.3} & \textbf{98.1} & \textbf{97.0} & \textbf{94.5} & \textbf{89.0} & \textbf{75.9} \\
\bottomrule
\end{tabular}%
}
\caption{MNIST test accuracy on all noise conditions trained by several baseline classifiers and after post-processed with NPC. \textbf{Bolded} with accuracy gain. Similar to the main paper, all experiments are replicated over five times with different random seeds. Reported results are mean value.}
\label{appen table: MNIST result}
\end{table}

\subsection{Comparison with Post-processors: Results on Real Dataset}
\label{appen:post_processing_real_dataset}
We compare the test accuracy of NPC and other post-processing method applicable after finishing training a classifier on Food-101 and Clothing-1M. We show NPC works best of all at Table \ref{tab:food and clothign post}.

\begin{table}[h]
\centering
\resizebox{0.9\textwidth}{!}{%
\begin{tabular}{c|cccc|cccc}
\toprule Dataset    & \multicolumn{4}{c|}{\textbf{Food101}}                    & \multicolumn{4}{c}{\textbf{Clothing1M}}                \\ \cmidrule(lr){0-1}\cmidrule(lr){2-5}\cmidrule(lr){6-9}
Method     & Classifier & KNN   & RoG   & NPC            & Classifier & KNN   & RoG   & NPC            \\ \midrule
CE         & 78.37      & 67.25$\pm$\scriptsize{0.2} & 78.38$\pm$\scriptsize{0.1} & \textbf{80.21}$\pm$\scriptsize{0.2} & 68.14      & 69.38$\pm$\scriptsize{0.1} & 68.05$\pm$\scriptsize{0.1} & \textbf{70.83}$\pm$\scriptsize{0.1} \\
Early Stop & 73.22      & 68.37$\pm$\scriptsize{0.2} & 75.87$\pm$\scriptsize{0.2} & \textbf{76.80}$\pm$\scriptsize{0.3} & 67.07      & 68.76$\pm$\scriptsize{0.0} & 68.25$\pm$\scriptsize{0.1} & \textbf{70.21}$\pm$\scriptsize{0.1} \\
SCE        & 75.23      & 64.72$\pm$\scriptsize{0.3} & 77.49$\pm$\scriptsize{0.1} & \textbf{78.26}$\pm$\scriptsize{0.3} & 67.77      & 69.34$\pm$\scriptsize{0.1} & 67.69$\pm$\scriptsize{0.2} & \textbf{70.36}$\pm$\scriptsize{0.1} \\
REL        & 78.96      & 78.90$\pm$\scriptsize{0.1} & \textbf{83.79}$\pm$\scriptsize{0.1} & 78.95$\pm$\scriptsize{0.4}          & 62.53      & 63.49$\pm$\scriptsize{0.0} & 66.12$\pm$\scriptsize{0.3} & \textbf{64.83}$\pm$\scriptsize{0.1} \\
Forward    & 83.76      & 83.68$\pm$\scriptsize{0.0} & 82.02$\pm$\scriptsize{0.1} & 83.77$\pm$\scriptsize{0.3}          & 66.86      & 69.40$\pm$\scriptsize{0.1} & 66.71$\pm$\scriptsize{0.2} & \textbf{70.02}$\pm$\scriptsize{0.1} \\
DualT      & 57.46      & 52.87$\pm$\scriptsize{0.2} & 60.82$\pm$\scriptsize{0.1} & \textbf{61.82}$\pm$\scriptsize{0.7} & 70.18      & 69.05$\pm$\scriptsize{0.1} & 69.46$\pm$\scriptsize{0.1} & 69.99$\pm$\scriptsize{0.4}          \\
TVR        & 77.34      & 65.98$\pm$\scriptsize{0.3} & 77.20$\pm$\scriptsize{0.1} & \textbf{79.37}$\pm$\scriptsize{0.1} & 67.18      & 68.27$\pm$\scriptsize{0.2} & 67.55$\pm$\scriptsize{0.1} & \textbf{69.44}$\pm$\scriptsize{0.1} \\
CausalNL   & 86.08      & 85.14$\pm$\scriptsize{0.1} & 85.93$\pm$\scriptsize{0.1} & \textbf{86.29}$\pm$\scriptsize{0.0} & 68.31      & 68.08$\pm$\scriptsize{0.1} & 68.37$\pm$\scriptsize{0.1} & \textbf{69.90}$\pm$\scriptsize{0.2} \\ \bottomrule
\end{tabular}%
}
\caption{Test accuracy after training the classifier with cross entropy loss, training KNN algorithm on representations, applying RoG, and applying NPC for \textit{Food101} (Food) and \textit{Clothing-1M} (Clothing).}
\label{tab:food and clothign post}
\end{table}

\newpage
\subsection{Flexibility of Prior Modeling for NPC}

As mentioned in Section \ref{section:implementation} of the main paper, we designed a prior distribution of the latent variable $y$ depending on $x$ as Eq. \ref{eq: prior} using predictions from KNN algorithm. However, the shape of a prior distribution does not have to be one-specific dimension sharp and all other dimensions flat. In this section, we report experiment results with more flexible modeling of a prior distribution as Eq. \ref{eq:topk prior}. Here, unlike the main page, we define $\Bar{Y}$ as the subset of the label set with $|\Bar{Y}|$ possible to be larger than 1 and $p$ as the prediction probability from KNN.

\begin{equation}
\begin{aligned}
\alpha_x^k& = \left\{\begin{matrix}
\delta & k\neq\Bar{y}\\ 
\delta+\rho \times p(\Bar{y}) & k=\Bar{y}
\end{matrix}\right. \;\;\text{for}\; \Bar{y}\in\Bar{Y},\; k=1,...,c
\end{aligned}
\label{eq:topk prior}
\end{equation}

Increasing flexibility of the prior may result in either accuracy increase or decrease; it can implicitly model information on the relations of classes by specifying the probability of a sample being assigned to each class, or inject noise and lose class information. We experiment the difference of a prior distribution for CIFAR-10 dataset with several noise conditions. Table \ref{tab:prior_top2} reports the test accuracy of NPC with its original prior and with the flexible version. We find out that TOP2 prior tends to work better than the original one with more severe noises or ASN or IDN noise condition. Since the prior distribution represents the information given from the input ($x$) for classification, we see the possibility of model development by modeling improved prior.

\begin{table}[h]
\centering
\resizebox{0.52\textwidth}{!}{%
\begin{tabular}{c ccccccccc}
\toprule
\multicolumn{1}{c}{\textbf{Model}} & \textbf{Clean} & \multicolumn{2}{c}{\textbf{SN}} & \multicolumn{2}{c}{\textbf{ASN}} & \multicolumn{2}{c}{\textbf{IDN}} & \multicolumn{2}{c}{\textbf{SRIDN}} \\
\cmidrule(lr){2-2}\cmidrule(lr){3-4}\cmidrule(lr){5-6}\cmidrule(lr){7-8}\cmidrule(lr){9-10}
& - & 20 & 80 & 20 & 40 & 20 & 40 & 20 & 40 \\ \midrule
CE & \textbf{89.0} & \textbf{80.8} & 17.0 & 84.7 & 78.8 & \textbf{80.9} & 59.9 & \textbf{74.3} & \textbf{64.3}\\
\rowcolor{Gray} w/ TOP2& 80.0   & 79.2   & \textbf{18.0}   & \textbf{85.3}   & \textbf{80.4}   & 80.3   & \textbf{64.6}   & 74.1   & 63.4   \\ \midrule
Joint & \textbf{84.4} & \textbf{80.2} & 8.3 & \textbf{83.0} & \textbf{77.7} & 80.7 & 69.1 & \textbf{72.0} & \textbf{63.6} \\ 
\rowcolor{Gray} w/ TOP2& 83.9   & 80.0   & \textbf{8.4}    & 82.4   & 77.3   & \textbf{81.1}   & \textbf{77.0}   & 71.3   & 62.8   \\ \midrule
Coteaching & 89.2 & \textbf{85.3} & \textbf{32.1} & 87.1 & 76.8 & \textbf{84.8} & \textbf{78.5} & \textbf{76.1} & \textbf{67.2} \\
\rowcolor{Gray}w/ TOP2&\textbf{89.3}   & 83.9   & 32.0   & \textbf{87.3}   & \textbf{77.4}   & 83.8   & \textbf{78.5}   & 75.9   & 67.0 \\\midrule
JoCoR & \textbf{89.3} & \textbf{86.0} & 27.0 & \textbf{85.1} & 79.0 & \textbf{85.8} & 80.1 & 75.9 & 66.7 \\
\rowcolor{Gray}w/ TOP2 & \textbf{89.3}   & 85.3   & \textbf{27.4}   & 85.0   & \textbf{80.0}   & 85.5   & \textbf{80.5}   & \textbf{76.1}   & \textbf{66.9} \\ \midrule
SCE  & 87.4 & 75.0 & 15.2 & 81.5 & 75.2 & 75.4 & 55.6 & 72.9 & 62.5 \\
\rowcolor{Gray}w/ TOP2 & \textbf{89.0}   & \textbf{79.2}   & \textbf{18.0}   & \textbf{85.3}   & \textbf{80.4}   & \textbf{80.3}   & \textbf{64.6}   & \textbf{74.1}   & \textbf{63.4}  \\ \midrule
Early Stop & \textbf{84.0} & 82.5 & \textbf{18.2} & 81.2 & 72.0 & 79.4 & 65.1 & \textbf{72.1} & \textbf{63.0} \\
\rowcolor{Gray}w/ TOP2 & 78.4   & \textbf{82.7}   & 17.9   & \textbf{81.3}   & \textbf{72.5}   & \textbf{79.7}   & \textbf{65.7}   & 71.7   & 62.8 \\ \midrule
LS  & \textbf{89.0} & \textbf{80.8} & 15.5 & 84.7 & 78.8 & \textbf{80.9} & 59.9 & \textbf{74.3} & \textbf{64.3}\\
\rowcolor{Gray}w/ TOP2 & \textbf{89.0}   & 79.2   & \textbf{18.0}   & \textbf{85.3}   & \textbf{80.4}   & 80.3   & \textbf{64.6}   & 74.1   & 63.4 \\ \midrule
REL & \textbf{83.4} & 78.6 & \textbf{26.0} & \textbf{75.9} & \textbf{76.1} & \textbf{78.5} & \textbf{51.2} & \textbf{70.7} & 64.2 \\
\rowcolor{Gray}w/ TOP2 & 81.6   & 77.0   & 25.6   & 74.6   & 75.6   & 77.0   & 51.1   & 70.2   & \textbf{64.3}\\ \midrule
Forward & \textbf{88.7} & 81.5 & 17.2 & 83.8 & 74.5 & 80.3 & 63.3 & \textbf{74.8} & 65.0 \\
\rowcolor{Gray}w/ TOP2 & 86.9   & \textbf{82.8}   & \textbf{20.8}   & \textbf{85.2}   & \textbf{78.3}   & \textbf{83.0}   & \textbf{76.4}   & 73.9  & \textbf{65.3} \\ \midrule
DualT & \textbf{86.0} & 83.0 & \textbf{8.4} & \textbf{83.0} & 77.5 & 81.0 & \textbf{77.3} & \textbf{70.1} & \textbf{64.0} \\
\rowcolor{Gray}w/ TOP2 & \textbf{86.0}   & \textbf{83.1}   & 8.3  & 82.0   & \textbf{78.1}   & \textbf{81.1}   & \textbf{77.3}   & 70.0  & 63.8  \\ \midrule
TVR & 88.3 & \textbf{80.8} & 15.7 & 84.1 & 76.5 & \textbf{80.8} & 60.7 & \textbf{74.5} & \textbf{64.5}\\
\rowcolor{Gray}w/ TOP2 & \textbf{88.4}   & 78.5   & \textbf{17.0}   & \textbf{84.3}   & \textbf{79.3}   & 79.7   & \textbf{65.0}   & 74.3  & 63.6 \\ \midrule
CausalNL & 89.7 & 81.2 & 18.8 & 85.0 & 74.8 & 81.2 & 71.9 & 75.3 & 63.9 \\
\rowcolor{Gray}w/ TOP2 & \textbf{90.6} & \textbf{81.4} & \textbf{19.0} & \textbf{85.7} & \textbf{75.4} & \textbf{81.5} & \textbf{72.5} & \textbf{75.8} & \textbf{64.0} \\ \bottomrule
\end{tabular}%
}
\caption{Test accuracy for CIFAR-10 datasets with various noisy label types. We demonstrate average performances with NPC of one-dimension sharp prior and its smoothed version (w/TOP2). TOP2 represents the dimension with largest probability and the second-best one become sharp with its value. The experimental results are averaged value over five trials. \textbf{Bolded} text denotes better performance.}
\label{tab:prior_top2}
\end{table}

\newpage
\subsection{Repetitive Application of NPC}
\label{sec:iter npc}

\begin{figure}[h]
    \centering
    \includegraphics[width=0.4\columnwidth]{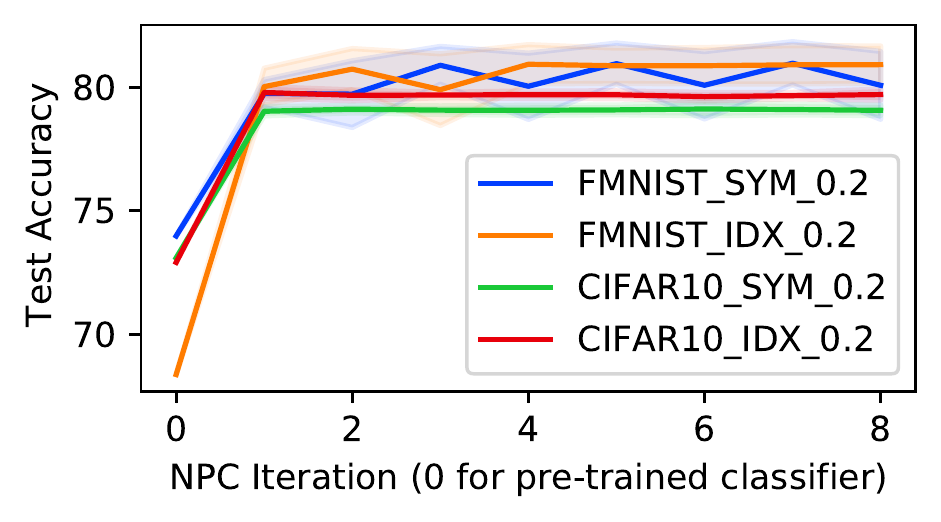}
    \caption{Iterative learning of NPC on CE for several datasets.}
    \label{fig:convergence}
\end{figure}

NPC calibrates the classifier prediction toward the true label as a post-processor. Having said that, It arises a new question : With the repetitive application of NPC, would the performance of given model gradually improve and finally torch the 100\% test accuracy, or the performance would converge to some extent? If it converges, when will it do? We implemented this experimental setting by designing the repetitive application of NPC, which iteratively utilizes calibrated prediction of NPC from previous iterations. Figure \ref{fig:convergence} shows that the NPC performance converges after the first deployment. It implies that NPC with a single iteration has already utilized enough information from the classifier to model latent true label.

\subsection{Extensive Comparison between CausalNL and NPC}

\begin{table}[h]
\centering
    \begin{tabular}{c|cc}
        \toprule Ratio & CausalNL & NPC\\ \midrule
        20$\%$ &81.81& \textbf{82.91} $\pm$ 0.1  \\
        40$\%$ &77.01& \textbf{78.83} $\pm$ 0.4 \\
        \bottomrule
    \end{tabular}
    \caption{Test accuracy on CIFAR-10 with IDN noise}
    \label{tab:resnet}
\end{table}

On the main page, we reported model accuracy of CausalNL lower than the one reported in the original paper \cite{causalNL}. We analyzed why this gap happened, found out that 1) the difference of backbone structure (9-layer CNN for ours, Resnet34 for the original paper) and 2) the difference of data normalization have affected the performance. we experiment CausalNL model under the same condition as the original authors did, and reproduce the performance as table \ref{tab:resnet}. Still, NPC shows better performance than CausalNL.

\subsection{Time Complexity Comparison on $T$ and $H$}
\label{append:time}
Our next question is: How long does it take on estimation of $T$ and $H$ to converge? We experimented it on MNIST Instance dependent 20 \% noise label dataset, and report the running time of each method on Table \ref{tab:time}.
\begin{table}[h]
    \centering
    \begin{tabular}{c|c}
        \toprule Model & Time \\ \midrule
        Forward & 636.4425 \\
        DualT & 3004.905 \\
        TVR & 468.7727 \\
        CausalNL & 4165 \\
        NPC & 28.2169 \\ \bottomrule
    \end{tabular}
    \caption{Model and Time spent to converge}
    \label{tab:time}
\end{table}


\end{document}